\documentclass{article}

\usepackage{microtype}
\usepackage{graphicx}
\usepackage{caption}
\usepackage[justification=centering]{subcaption}
\usepackage{bbm}
\usepackage{booktabs} % for professional tables
\usepackage{multirow}
\usepackage{enumitem} %for no spaces between items 

\usepackage{hyperref}

\usepackage[arxiv]{arxiv_style}

\usepackage{amsmath}
\usepackage{amssymb}
\usepackage{mathtools}
\usepackage{amsthm}

\usepackage[capitalize,noabbrev]{cleveref}

\theoremstyle{plain}
\newtheorem{theorem}{Theorem}[section]
\newtheorem{proposition}[theorem]{Proposition}

\theoremstyle{definition}

\theoremstyle{remark}

\arxivtitlerunning{Learning to Fuse Temporal Proximity Networks: A Case Study in Chimpanzee Social Interactions}

\begin{document}

\twocolumn[
\arxivtitle{Learning to Fuse Temporal Proximity Networks:\\A Case Study in Chimpanzee Social Interactions} 
\arxivsetsymbol{equal}{*}

\begin{arxivauthorlist}
\arxivauthor{Yixuan He}{asu}
\arxivauthor{Aaron Sandel}{uta}
\arxivauthor{David Wipf}{Amazon}
\arxivauthor{Mihai Cucuringu}{UCLA,UnivOx}
\arxivauthor{John Mitani}{umich}
\arxivauthor{Gesine Reinert}{UnivOx,turing}
\end{arxivauthorlist}

\arxivaffiliation{asu}{School of Mathematical and Natural Sciences, Arizona State University, Phoenix, AZ, United States}
\arxivaffiliation{uta}{Department of Anthropology, University of Texas at Austin, Austin, TX, United States}
\arxivaffiliation{UCLA}{Department of Mathematics, University of California Los Angeles, Los Angeles, CA, United States}
\arxivaffiliation{UnivOx}{Department of Statistics, University of Oxford, Oxford, United Kingdom}
\arxivaffiliation{Amazon}{Amazon Web Services AI Shanghai Lablet, Shanghai, China}
\arxivaffiliation{umich}{Department of Anthropology, University of Michigan, Ann Arbor, MI, United States}
\arxivaffiliation{turing}{The Alan Turing Institute, London, United Kingdom}

\arxivcorrespondingauthor{Yixuan He}{Yixuan.He@asu.edu}

\arxivkeywords{Machine Learning, Network Time Series, Time Series Modeling, Network Analysis, Animal Social Networks, Node Similarity, Clique Detection}

\vskip 0.3in
]


% The \arxivEqualContribution command is standard text for equal contribution.

\printAffiliationsAndNotice{}%\arxivEqualContribution} % otherwise use the standard text.

\begin{abstract}
How can we identify groups of primate individuals which could be conjectured to drive social structure? To address this question, one of us has collected a time series of data for social interactions between chimpanzees. Here we use a network representation, leading to the task of combining these data into a time series of a single weighted network per time stamp, where different proximities should be given different weights reflecting their relative importance. We optimize these proximity-type weights in a principled way, using an innovative loss function which rewards structural consistency for consecutive time steps. The approach is empirically validated by carefully designed synthetic data. Using statistical tests, we provide a way of identifying groups of individuals that stay related for a significant length of time. Applying the approach to the chimpanzee data set, we detect cliques in the animal social network time series, which can be validated by real-world intuition from prior research and qualitative observations by chimpanzee experts.
\end{abstract}

\section{Introduction}
\label{sec:introduction}
What drives the social structure in primates? To address this, one of us has collected a rich data set on chimpanzees in Uganda. Starting in 1998, the data set records proximities between chimpanzees when observed. Due to visibility issues in the tropical forest and limits on the time that researchers can spend in the forest recording chimpanzee data, only a small number of these proximities have been observed. Surprisingly, though these chimpanzees remained in relatively stable social structures initially, there is a dramatic change from 2014 in the population structure, which has attracted our attention. To analyze these social proximity data and to further explore what leads to the change, we choose a network representation; a key contribution of this paper is a principled way of obtaining such a representation. 

Social networks are important in social and biological sciences, where individuals are treated as nodes connected by some interactions which are treated as possibly weighted edges between nodes; see, for example, \citet{wasserman1994social}. Creating networks based on various interactions is useful for modeling a range of dynamics, such as disease spread and information transfer. Networks are also valuable for determining social structures. Although social networks have featured prominently in sociology and, in the last two decades, animal behavior~\citep{pinter2014dynamics}, several challenges persist: What behaviors should be used to construct networks? How are appropriate weights determined? 
These issues are compounded when data interactions of different types are available; here we are thinking of many proximity records with different proximity ranges.

Multilayer networks are one possibility for representing such records, as in  \citet{kivela2014multilayer}, but it may not be easy to interpret and may obfuscate that different proximities are related.
The situation becomes even more intricate given a time series of proximity data, as in the chimpanzee case study which motivates our work. The animals form a variety of social groupings that change throughout the day. The same is true for other species with fission-fusion social dynamics~\cite{silk2014importance, ramos2018quantifying}. A biological research question is then to identify groups of individuals close to each other at multiple times; in a human data set, one might interpret these as friendship groups \cite{sekara2014strength}. For example, \citet{stopczynski2018physical} observed that for human proximity data, close contacts (within about 1 m) have different characteristics than long-range contacts (within around 10-15 m); close contacts tend to be friends, whereas long-range contacts are more likely to be chance encounters. They construct two different networks for these data; our approach provides a combined (``fused") network instead, by attaching importance weights to various proximity levels.

There are different approaches for exploring potential driving forces for social structures~\cite{roehner2007driving}, but here we focus on strong dyadic bonds, motivated by a baboon example~\cite{lerch2021better}. We explore potential structural drivers from a novel angle of long-term close relationships, addressing several methodological questions: (1) How to combine the available proximity data into one network per time step? 
(2) How do we identify groups of individuals close to one another more often than expected by chance?

Traditional methods for constructing animal social networks often rely on statistical models~\cite{farine2015constructing,brask2024introduction}, or on reference distributions obtained via randomizations \cite{hobson2021guide}. Instead, we take inspiration from network/graph optimization to model relationships between nodes/entities~\cite{Zhou, he2024robust}, by devising a novel loss function which we optimize. By leveraging different levels of proximity data collected over extended periods, our novel optimization approach is designed to capture the underlying consistency of social relationships. In our case, the \emph{first} question involves constructing a network time series with a single weighted network per time step (instead of, e.g., a multiplex network with 10 layers and $10\times 9/2$ user-defined interlayer weights, which would lead to a large number of parameters to choose). To our knowledge, there is \textbf{no existing work on how to combine multiple networks with the special feature of representing nested levels of proximity into a single network}. Here, we apply optimization concepts to create social networks of chimpanzees by combining networks with nested proximity levels and comparing them over time. Given the dynamic nature of chimpanzee groups, the resulting time series of networks provides a useful system to address the \emph{second} methodological question. 

Our key contributions are summarized as follows:
\begin{itemize}[noitemsep, nolistsep, leftmargin=*]
    \item We propose a novel optimization pipeline to represent real-world multi-level proximity data using networks. The construction pipeline is empirically validated by a carefully designed set of synthetic data.
    \item We provide theoretical contributions on two novel notions of individual similarities across time based on sequences of Bernoulli trials with evolving success rates. The analysis can be utilized to detect long-term close relationships.
    \item We apply the novel network combination/fusion pipeline to present an effective network time series representation of proximity data in a wild chimpanzee social group across time. Using additionally the similarity notions, we show that there is sufficient information in these data to identify groups of chimpanzees which stay in each other's wider community for a surprisingly large amount of time. 
\end{itemize}
%%%%%%%%%%%%%%%%%%%%%%%%%%%%%%%%
\section{Literature Review}

Combining or fusing networks/graphs that represent different views into a unified structure is a well-studied problem relating to our proximity network fusion setting. \citet{kang2020multi} proposed a multi-graph fusion model that simultaneously performs graph combination and spectral clustering. The idea that multi-view data admits the same clustering structure aligns with the structural consistency assumptions in our proposed method. However, this method, together with other multi-view fusion approaches such as \citet{yang2019adaptive} and \citet{yang2024bidirectional}, cannot naturally consider structural consistency for consecutive time steps. \citet{zhang2021adaptive} and \citet{hu2022spatio} leveraged spatio-temporal fusion to address challenges in predicting remaining useful life and in trajectory data analytics, respectively. However, their fusion modules, being within a neural network architecture, cannot be readily applied to combining multiple levels of proximity, as in our case.

Node similarity in network time series is another key topic in our paper. \citet{gunecs2016link} considered neighborhood-based node similarity scores for link prediction. \citet{yang2019time} developed a diffusion model to drive the dynamic evolution of node states and proposed a novel notion of dissimilarity index. The approach in \citet{meng2018coupled} learns node similarities by incorporating both structural and attribute-based information. However, none of them considers node similarity based on long-term close relationships.

Detecting long-term relationships in dynamic systems has been explored from various perspectives. In long-term studies of primate relationships, scientists report how long certain pairs have a high frequency of interaction, but they do not provide a statistical approach to determine what constitutes a persisting relationship. For example,  \citet{mitani2009male} emphasized the significance of long-term affiliative relationships in social mammals by considering pairwise affinity indexes between male dyads~\cite{pepper1999general} to quantify long-term relationships, but the indexes neglect consecutive proximities in the time series. \citet{derby2024female} presented a Bayesian multimembership approach to test what factors predict the persistence of proximity relationships, but persistence is defined deterministically by being in proximity for more than a fixed time length without considering hypothesis-testing based on an expected duration. \citet{qin2019mining} investigated the interplay between temporal interactions and social structures, but their approach is constrained to analyzing periodic behaviors. \citet{escribano2023stability} studied the stability of personal relationship networks in a longitudinal study of middle school students by exploring the persistence of circle structures, which is different from our novel perspectives of node similarity. Finally, our task is not to find consistent communities/close relationships across time but to identify pairs of individuals who stay in the same community. Methods enforcing community consistency over time, like \citet{mucha2010community}, may unintentionally keep pairs together, which could obfuscate the signal of interest. 

Our real-world data pose particular challenges, involving a focal-based biased data collection protocol with a specific hierarchy related to proximity, which makes previous network construction methods and node similarity analysis not directly applicable, resulting in no proper existing baselines. For example, to construct networks from observations, \citet{rabbat2008network} modeled co-occurrence observations via a random walk, and \citet{casiraghi2017relational} assumed a higher-dimensional configuration model, but our data have a clear focal bias, which both papers overlook. \citet{psorakis2012inferring} took a spatial approach around gathering events, but following a focal male is not a gathering event. In \citet{abraham2015low}, the network layers represent separate categories, and \citet{newman2018network} assumed independent edge measurements, but in our data, the proximities are nested. As locations were not recorded in our data, trajectory-based analysis, e.g., \citet{han2021graph}, is not applicable. Due to the data collection bias, temporal embeddings via graph learning, e.g., \citet{haddad2019temporalnode2vec}, edge survival models, e.g., \citet{ccelikkanat2024continuous}, or temporal exponential random graph models, e.g., \citet{leifeld2019theoretical}, are also not appropriate. Community tracking methods, e.g.,  \citet{mazza2023modularity}, are not proper baselines either, as they cannot naturally capture dyadic relationships. Due to data collection bias and missing observations, directly aggregating edge weights or counting common neighbors would be inappropriate. Instead, we apply a community detection method to mitigate the focal bias. 
%%%%%%%%%%%%%%%%%%%%%%%%%%%%%%%%
\section{Motivation: Chimpanzees in Uganga}
\label{sec:chimp_data_description}
\textbf{Notations.} Denote a static undirected weighted network (graph, used interchangeably) as $\mathcal{G}=(\mathcal{V}, \mathcal{E}, w),$ with $\mathcal{V}$ the set of nodes, $\mathcal{E}$ the set of edges encoding node relationships, and $w \in [0, \infty)^{| \mathcal{E}|}$ the set of edge weights. Such a network can be represented by the adjacency matrix $\mathbf{A} = (A_{ij})_{i,j \in \mathcal{V}}$,  
with $\mathbf{A}_{ij}=0$ if no edge exists from $v_i$ to $v_j$; if there is an edge $e$ between $v_i$ and $v_j$, we set $A_{ij} = w_{e}$, the edge weight. A \emph{network time series} is a time series of networks,  $\{\mathcal{G}^{(t)}\},$ where for each time step $t,$ the network is static. At each time step, a \emph{multiplex network} is constructed with each type of interaction between the nodes being described by a single-layer network; the different layers of networks describe the different modes of interaction.
We denote the multiplex network time series with $H$ layers by $\{\mathcal{G}^{(h, t)}\},$ where $h\in\{1, \dots, H\}$ refer to the different types/layers.

\textbf{Data Description.} Our methods are motivated by a unique data set on chimpanzees one of us collected over nearly three decades, which was manually collected by highly trained observers following a rigorous protocol. We received clearance from relevant animal care and use committees due to the purely observational nature of the research. This data set is composed of observations of a community of wild chimpanzees in Ngogo, Kibale National Park, in Uganda. Long-term research began at Ngogo in 1995, and chimpanzees, especially adult males, were habituated to human observers starting in 1998. The social relationships of these chimpanzees have been studied for 30 years, see for example \citet{langergraber2017group}. The chimpanzees at Ngogo were part of one social group, but chimpanzees exhibit ``fission-fusion'' social dynamics, so they are rarely, if ever, in the same place at the same time. Instead, they form temporary associations that change throughout the day. Prior studies have identified various structures within chimpanzee groups~\cite{badihi2022flexibility, mitani2003social}. In our case, the total number of individuals varies from 150 to 200 per year. The data set we use in this paper relies on 24 years of data on 219 individuals (77 adult male chimpanzees who were the focus of behavioral observations, and 142 additional chimpanzees--20 males and 122 females--that were not the focus of direct observation, but interacted with the focal subject during observations). Data were collected every year except 2020, when COVID-19 restrictions prohibited this activity.
Field seasons typically occurred during two or three consecutive months per year. Observations targeted a subset of adult male chimpanzees. The behavioral observation procedure involved following one ``focal" male for an hour and recording three main social interactions: (1) ``party": all chimpanzees that were in social/spatial association (within roughly 100m of the focal subject during the hour session); (2) ``proximity": chimpanzees within physical proximity of the focal subject (within 0-2m or 2-5m) recorded at 10-minute intervals; and (3) ``grooming": chimpanzees involved in grooming with the focal subject during the hour sampling period (which should be within 2m of the focal subject). After an hour of observation, another focal subject was selected. Several other behaviors were also recorded (e.g., territorial patrols, hunting, aggression, self-grooming). This is a standard sampling method in primate behavior, see \citet{altmann1974observational}. 

\textbf{Proximity Types.} The data set contains information on different types of proximities. The proximity types relate to distances, grooming, and whether the proximity is to a focal male or just in the vicinity of a focal male. For example, we denote by ``prox2" and ``prox5" that an individual is within 2m (two meters) or 2-5m of a certain focal chimpanzee, respectively; ``party" denotes an individual in the social/spacial association (within roughly 100m) of a focal subject but not within 5m distance. In Fig.~\ref{fig:proximity_types}, we illustrate different types of proximity levels. Here we use F to abbreviate a focal male, A, B, and C are in ``prox2" to F, D, E, G, and H are in ``prox5" to F, while the rest are in ``party" to F. In particular, B is grooming with F. In total, we obtain ten proximity types as detailed in Appendix (App.)~\ref{app_sec:proximity_types}. 

A natural question then arises: \emph{How do we use these data to construct informative networks for the entire group}? 

\begin{figure}[ht]
\centering
\includegraphics[width=0.6\linewidth]{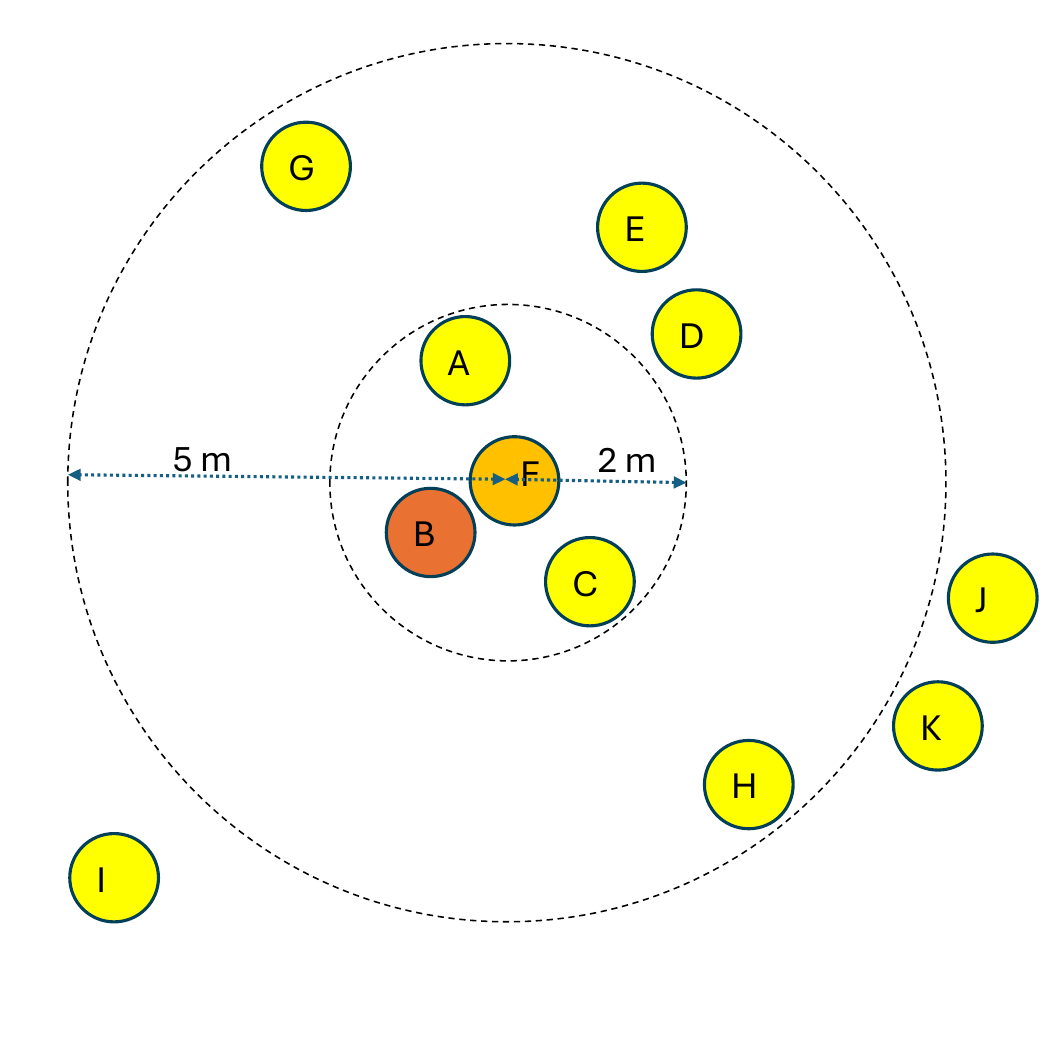}
    \caption{Proximity in chimpanzees. F is a focal male, A, B, and C are in ``prox2" to F, D, E, G, and H are in ``prox5" to F, while the rest are in ``party" to F. B is grooming with F.}
    \label{fig:proximity_types}
\end{figure}
As the recorded data is binary and there are various types of interactions (proximities), we start with a multiplex network representation with layers representing proximity levels. Based on these various levels of proximity, we construct networks based on single relationships, where every edge is given a unit edge weight for a single day if on that date at least one occurrence of that type is observed. The yearly edge weight is computed as the sum of the edge weights (0 or 1) obtained from all dates involved. As a result, for each year, we obtain 10 weighted single-relationship networks, one for each type. An example is provided for August 2006 in Fig.~\ref{fig:separate_graphs200608} in App.~\ref{app_sec:extended_plots_and_tables}. In this example, isolated nodes are omitted in the plots for the 10 single-relationship networks.

Pairs of chimpanzees may stay in similar locations in consecutive 10-minute intervals. We hypothesize that consecutive occurrences may contribute to an additional intensity of interaction/proximity. Here, we regard same-day occurrences as consecutive ones. Therefore, for each single-relationship network, we construct an ancillary network to specifically record multiple occurrences on the same day. For each ancillary network, an edge is added for each day if multiple occurrences are observed, with the number of consecutive occurrences minus one as the daily weight, while the total edge weight is the sum of the daily weights. We can view the raw networks as counting the number of days in a year when a certain proximity is observed between two individuals, while the ancillary networks count how many times multiple occurrences are observed within this year, minus the number of days, so adding both networks together counts the total number of occurrences.

Another question naturally follows: \emph{Can we fuse these networks effectively to describe the time series of chimpanzee proximities?} Here, ``fuse" refers to combining related layers in the multiplex network into a single network; ``effectively" means preserving proximity and hierarchy information while achieving small error, assuming an underlying ground-truth network representation exists.

Indeed, we observe inherent hierarchies in the proximity levels based on the upper bound of pairwise differences described in each type, indicating that not every type of interaction should be given the same weights. Hence, we propose an approach that learns weights for each type of interaction and thus obtains a weighted network as a representation.

\section{A Parametric Network Model}
Suppose the ground-truth network $\mathcal{G}^{(t)}$ for each time step $t$ can be expressed by its adjacency matrix $\mathbf{A}^{(t)}$, then with $H$ hierarchies of proximity levels ($H=10$ for chimpanzees), we have $\mathbf{A}^{(t)} = \sum_{h=1}^H W_h \mathbf{A}^{(h, t)}$ with positive nondecreasing weights $W_h$'s, i.e., $0 < W_{h_1} \leq W_{h_2}$ if $h_1\leq h_2.$ Given the considerations for consecutive occurrences, we further split each $\mathbf{A}^{(h, t)}$ into a weighted sum of the raw network, $\mathbf{A}^{(\text{raw}, h, t)},$ and the ancillary (``add'') network, $\mathbf{A}^{(\text{add}, h, t)}.$ In the chimpanzee example, the ancillary/add networks represent ``consecutive'' (i.e., same-day) occurrences.

For the purpose of modeling, we further express the nondecreasing weights as a sum of nonnegative increment weights, $W_h=\sum_{j=1}^h w_j,$ resulting in
{\small\begin{equation}
    \begin{aligned}
    \label{eq:adj_t}
    \mathbf{A}^{(t)} = \sum_{h=1}^H W_h \mathbf{A}^{(h, t)}=\sum_{h=1}^H (\sum_{j=1}^h w_j) (\mathbf{A}^{(\text{raw}, h, t)} + w_\text{add}\mathbf{A}^{(\text{add}, h, t)}),
    \end{aligned}
\end{equation}}
so that $W_h$ is the network combination weight for the $h$-th hierarchy, $w_j=W_j-W_{j-1}\geq 0$ is the nonnegative increment from $W_{j-1}$ to $W_{j},$ $w_\text{add}\in[0,1]$ is the increment weight for a single network type when increments are considered, e.g., when consecutive occurrences of a certain type of proximity are recorded, $\mathbf{A}^{(h, t)}$ is the adjacency matrix for a single network type $h$ at time step $t$, $\mathbf{A}^{(\text{raw}, h, t)}$ is the raw adjacency matrix without considering added increments, and $\mathbf{A}^{(\text{add}, h, t)}$ records the increments for each type. Here we set $w_1=W_1=1$ fixed by default for normalization. Note that we assume that $\mathbf{A}^{(\text{raw}, h, t)}$ and $\mathbf{A}^{(\text{add}, h, t)}$ are known, while $w_j$ and $w_\text{add}$ are learnable parameters. 
%%%%%%%%%%%%%%%%%%%%%%%%%%%%
%%%%%%%%%%%%%%%%%%%%%%%%%%%%
\section{Proposed Method: ProxFuse}
\label{sec:proposed_method}
To learn the combination weights, one possibility is to use deep neural networks such as a temporal graph neural network~\cite{rozemberczki2021pytorch} and make fusing the networks as part of the end-to-end training procedure by conducting regular tasks on graphs. However, in our initial attempts utilizing temporal graph neural networks, even though the prediction error can be reduced to almost zero, the parameters of interest are not properly optimized. This is probably due to overparametrization, with only a tiny proportion of all parameters being what we truly want to optimize. For example, for our chimpanzee data set, TGCN~\cite{zhao2019t} can easily require above 1000 parameters, far more than the ten combination weights that we actually need to learn. In general, deep neural networks have far more parameters than those of interest, making it hard to optimize the key parameters (here, the combination weights of different layers in the multiplex network). Hence, we propose a simple optimization model to fuse proximity networks (termed \emph{ProxFuse}) to learn the increment weights, where trainable parameters are primarily network combination weights. Inspired by chimpanzee experts, we propose an objective based on relative structural consistency for consecutive time steps, involving both the global network structure and local connection strength patterns. Specifically, we assume that the adjacency matrix may evolve, but the underlying similarity structure between nodes (based on which the adjacency matrix is generated, but with possible magnitude fluctuations) and the relative magnitudes of weighted node degrees should stay relatively consistent with only mild updates for consecutive time steps. In this real-world application, the networks are not fixed in time. The loss function penalizes dramatic changes, but a gradual evolution of the network structure is possible, and indeed observed.

Let $\mathcal{N}^{(t)}$ denote the set of nodes that exist (i.e., with nonzero degree) at time step $t\in \{ 1, 2, \ldots, T\}$; we assume $T \geq 2$ so that we have at least two time steps.
For each time step $t$, we first extract the set of nodes $\mathcal{N}^{(t, t+1)}=\mathcal{N}^{(t)}\cap\mathcal{N}^{(t+1)}$ which co-exist at time steps $t$ and $t+1$. We then construct subnetworks based on $\Tilde{\mathbf{A}}^{(t, t, t+1)}=\mathbf{A}^{(t)}_{\mathcal{N}^{(t, t+1)}, \mathcal{N}^{(t, t+1)}}$ and $\Tilde{\mathbf{A}}^{(t+1, t, t+1)}=\mathbf{A}^{(t+1)}_{\mathcal{N}^{(t, t+1)}, \mathcal{N}^{(t, t+1)}}$; both matrices take values in $\mathbb{R}^{\left\lvert\mathcal{N}^{(t, t+1)}\right\rvert\times \left\lvert\mathcal{N}^{(t, t+1)}\right\rvert}$. For each network, we construct a similarity graph $\mathbf{S}^{(t, t, t+1)}$ based on the adjacency matrix $\Tilde{\mathbf{A}}^{(t, t, t+1)}$ by taking into account the node similarities. Indeed, we treat each row $\Tilde{\mathbf{A}}^{(t, t, t+1)}_i$ in $\Tilde{\mathbf{A}}^{(t, t, t+1)}$ as a feature vector for node $i$ at time $t$, and compute the cosine similarity values between individuals, i.e., $\mathbf{S}^{(t, t,t+1)}_{i,j}=\frac{\Tilde{\mathbf{A}}^{(t, t, t+1)}_i \cdot \Tilde{\mathbf{A}}^{(t,t, t+1)}_j}{\left\lVert \Tilde{\mathbf{A}}^{(t, t, t+1)}_i\right\rVert_2 \left\lVert \Tilde{\mathbf{A}}^{(t, t, t+1)}_j\right\rVert_2},$ where the numerator takes the vector dot-product and $\lVert \cdot\rVert_2$ denotes the vector 2-norm. We can similarly compute $\mathbf{S}^{(t+1, t,t+1)}_{i,j}=\frac{\Tilde{\mathbf{A}}^{(t+1, t, t+1)}_i \cdot \Tilde{\mathbf{A}}^{(t+1,t, t+1)}_j}{\left\lVert \Tilde{\mathbf{A}}^{(t+1, t, t+1)}_i\right\rVert_2 \left\lVert \Tilde{\mathbf{A}}^{(t+1, t, t+1)}_j\right\rVert_2}.$ We treat the similarity matrix as the underlying network generator for edges; so it should stay relatively consistent for consecutive time steps. Based on $\mathbf{S}^{(t, t, t+1)}$ and $\mathbf{S}^{(t+1, t, t+1)},$ one objective is to minimize 
{\small\begin{align}
\label{eq:loss_similarity} \mathcal{L}_\text{sim}=\frac{1}{T-1}\sum_{t=1}^{T-1}\sum_{i,j\in\mathcal{N}^{(t,t+1)}}\Big(\mathbf{S}^{(t, t, t+1)}_{i,j}-\mathbf{S}^{(t+1, t, t+1)}_{i,j}\Big)^2.
\end{align}}
In addition, we assume that weighted node degrees remain relatively stable for consecutive time steps. We compute normalized weighted node degrees for time step $t$ as $d_i^{(t,t,t+1)}=\frac{\sum_j\Tilde{\mathbf{A}}^{(t,t, t+1)}_{i,j}}{\sum_{j,k}\Tilde{\mathbf{A}}^{(t,t, t+1)}_{j,k}}.$ Likewise, we can compute $d_i^{(t+1,t,t+1)}.$ We then obtain another term of loss function as
{\small\begin{align}
    \label{eq:loss_deg}
 \mathcal{L}_\text{deg}=\frac{1}{T-1}\sum_{t=1}^{T-1}\sum_{i\in\mathcal{N}^{(t,t+1)}}\left(d^{(t, t, t+1)}_{i}-d^{(t+1, t, t+1)}_{i}\right)^2.
\end{align}}
To penalize extreme combination weights, we additionally add a regularization term as
{\small\begin{equation}
    \label{eq:loss_reg}
    \mathcal{L}_\text{reg}=\frac{1}{H}\left(\lVert w_\text{add}\rVert_2^2 + \sum_{h=2}^{H}\lVert w_h\rVert_2^2\right).
\end{equation}}

To summarize, our optimization loss function amounts to
\begin{equation}
    \label{eq:total_loss}
    \mathcal{L}=\alpha_1\mathcal{L}_\text{sim} + \alpha_2\mathcal{L}_\text{deg} + \alpha_3\mathcal{L}_\text{reg}.
\end{equation}
The values of $\alpha_1, \alpha_2,$ and $\alpha_3$ are considered to be hyperparameters. Since we normalize the similarity loss and the degree loss, we expect these loss values to be balanced, and choose an equal factor of $\alpha_1=\alpha_2=1$.  $\alpha_3$ is considered a small nonzero value for regularization, and we set $\alpha_3=0.001$ by default. We do not conduct hyperparameter selection/ablation studies, as both main loss terms need to be considered and balanced to model relative consistency.

To cope with the nonnegativity requirement of the learnable network combination weights, $w_h$'s, we employ the inverse of the softplus function $\Tilde{w}_h=\log(\exp(w_h)-1)$ to the initial values of the $w_h$'s, replacing $w_h=0$ by $w_h=0.0001$ for numerical stability. We then employ the softplus function $w_h=\log(1+\exp(\Tilde{w}_h))$ to transform $\Tilde{w}_h\in\mathbb{R}$ back to $w_h>0,$ and set $w_h=0$ for tiny $w_h$ to ensure $w_h\in[0, \infty)$. For $w_\text{add}\in(0, 1)$ (similarly treating $0$ as $0.0001$ and $1$ as $0.9999$), we apply a logit transformation with $\Tilde{w}_\text{add}=\log\left(\frac{w_\text{add}}{1-w_\text{add}}\right)\in\mathbb{R}$, whose inverse function is the logistic function $w_\text{add}=\frac{1}{1+\exp(-\Tilde{w}_\text{add})}.$

We split the time series into training, validation, and test sets, where validation is used for early stopping, and the test set is used for learned model comparison and to select the ``best" optimized set of parameters based on the lowest test loss value (setting $\alpha_3=0$ in Eq.~\eqref{eq:total_loss} during selection).

To address the possible influence of parameter initialization, we run multiple initializations to obtain estimated optimal sets of the network combination weights. One heuristic of setting different initializations is: (1) uniform initial values from $[0.1, 0.2, 0.5, 1, 2, 5]$ (note that $w_\text{add}\leq1$ so we cap larger initial values to 1); (2) one parameter is initialized to be 1 and the rest initialized to be 0.1.
%%%%%%%%%%%%%%%%%%%%%%%%%%%%
%%%%%%%%%%%%%%%%%%%%%%%%%%%%%%%%
%%%%%%%%%%%%%%%%%%%%%%%%%%%%
\section{Experiments}
\label{sec:experiments}
As discussed in Literature Review, \emph{no existing method is directly applicable to our task}. To validate the efficacy of our novel network combination method, we test our proposed method on the chimpanzee data, and construct synthetic models with known combination weights for further empirical evidence. Importantly, these kinds of data sets involve enormous human labor across years, typically representing lifetime efforts, and hence they tend to be private, e.g., \citet{morrison2021rapid, derby2024female, badihi2022flexibility}. Therefore, while other real-world data sets exist, we do not have direct access to them, but researchers with those data sets may benefit by building upon our work here and adapting the pipeline to their data. Meanwhile, open-source data sets such as \citet{franz2015knockouts} and \citet{lusseau2003bottlenose} are not suitable here due to their lack of hierarchies in proximity levels recorded in the data sets or the lack of focal bias. Given its uniqueness, the chimpanzee data set is in itself rich enough to warrant a detailed study. A thorough experimental setup is provided in App.~\ref{app_sec:experiment_setup}.
%%%%%%%%%%%%%%%%%%%%%%%%%%%%
\subsection{A Carefully Designed Synthetic Network Model}
The synthetic model fixes the same cumulative weighted network $\mathcal{G}^{(t)}=(\mathcal{N}^{(t)},\mathcal{E}^{(t)})$ for each time step; the model assigns the edges to each hierarchy layer $h$ and to the two matrices, $\mathbf{A}^{(\text{raw}, h,t)}$ and $\mathbf{A}^{(\text{add}, h,t)}$. Different assignments then produce potentially different layer-specific edge weights. These assumptions are made as a sanity check to ensure that ProxFuse is able to achieve near-zero similarity and degree loss values. We let $n$ be the number of nodes in total, $T$ be the number of time steps, $H$ be the number of hierarchies (e.g.,  different proximity levels), and $\{p_h\}_{h=1}^H$ edge probabilities for each layer of the multiplex network. Details of the construction are provided in App.~\ref{app_sec:synthetic_generation}.
%%%%%%%%%%%%%%%%%%%%%%%%%%%%%%%%%%%%
\subsection{Synthetic Data Empirical Results}
We conduct experiments on multiple synthetic data sets. Here we set $p_h=0.1$ throughout, $n=100, p_\text{add}=0.1, T=14$ with training:validation:test=8:3:3 in the split. We take $H=5,$ vary $w_\text{add}\in\{0, 0.3\},$ and take different $w_h$'s. We compare the results for $\alpha_3=0$ and $\alpha_3=0.001.$ 

With $\alpha_3$ fixed, the final optimized weights are typically robust to initialization. In general, with $\alpha_3=0,$ for our synthetic data, the proposed method can perfectly recover the combination weights of interest up to one decimal point. With regularization considered ($\alpha_3=0.001$), the final estimated values are typically slightly smaller for actual nonzero parameters, while sometimes more parameters of interest are urged to take nonzero values. The set of parameters and optimized values based on the lowest test loss is provided in Tab.~\ref{tab:synthetic_res} in App.~\ref{app_sec:extended_plots_and_tables}. To conclude, our carefully designed synthetic data sets verify that ProxFuse can recover combination weights from ground truth with robustness to initialization.
%%%%%%%%%%%%%%%%%%%%%%%%%%%%%%
\subsection{Application to Chimpanzee Networks}
\label{sec:chimp_network_construction}
Applying ProxFuse to our chimpanzee data, we optimize the network combination weights by assuming a relatively stable normalized similarity matrix and relatively stable normalized weighted degrees for consecutive years. As chimpanzee researchers observed major changes in the chimpanzee social interactions from 2014, we conduct experiments for the period 1998 to 2012, leaving a one-year gap for the change to happen. We use 1998-2005 as the starting year for the training set, 2006-2008 as the starting year of the validation set, and 2009-2011 as the starting year of the test set. 

From our experiments, no matter which initialization we use, the optimal set of parameters, with one decimal point, is mostly robustly optimized to: $w_1=1.0, w_2=w_3=w_4=0.0, w_5=4.7, w_6=1.3, w_7=1.6, w_8=2.0, w_9=0.0, w_{10}=0.1,$ and $w_\text{add}=0.0.$ Although one initialization produces $w_{10}=0.2,$ this produces a slightly larger test loss value (setting $\alpha_3=0$ in Eq.~\eqref{eq:total_loss} for tests) as $0.01063\pm 0.00000,$ compared to the others, which obtain $0.01061\pm 0.00000$. The robustness of the final set of learned parameters to initialization empirically validates the efficacy and robustness of ProxFuse. This gives us network combination weights $W_1=W_2=W_3=W_4=1.0, W_5=5.7, W_6=7.0, W_7=8.6, W_8=W_9=10.6, W_{10}=10.7,$ and addition parameter $w_\text{add}=0.0.$

The combination of final learned weights implies that the first four types of proximity involving ``party" individuals should roughly be treated the same, possibly due to the high variance in possible actual distances between individuals (see J \& K in Type 1 and F \& I in Type 4 in Fig.~\ref{fig:proximity_types} for example). With the leading magnitude in $w_5,$ we observe the large proximity gap between ``party" individuals and those within 5m from the focal subject. We can also explore proximity differences from the nontrivial gaps between $W_5, W_6, W_7,$ and $W_8.$ With $w_9=0,$ we conclude that being within a circle with a radius of 2m is probably already a very close relationship. The small addition of $w_{10}$ to grooming indicates that individuals are actually more closely related if they groom each other. Finally, $w_\text{add}=0$ indicates that consecutive occurrences are normally recorded only because chimpanzees may stay in the same place for a while, instead of corresponding to another level of increased proximity.

For completeness of evaluation, if we set $\alpha_3=0$ during training, we obtain an even smaller optimized average test loss of 0.007536 with a uniform initialization of all parameters to be 5.0 and $w_\text{add}$ initialized to 0.1. The final result is $w_1=1.0, w_2=w_3=w_4=0.0, w_5=39.7, w_6=w_7=0.0, w_8=40.0, w_9=w_{10}=0.0,$ and $w_\text{add}=0.0,$ resulting in $W_1=W_2=W_3=W_4=1.0, W_5=W_6=W_7=40.7, W_8=W_9=W_{10}=80.7,$ and addition parameter $w_\text{add}=0.0.$ This implies that the most notable proximity gap comes from being within 5m of the focal subject and from being within 2m of the focal subject. These two nontrivial gaps align with the two biggest learned increments when we set $\alpha_3=0.001$, i.e., $w_5$ and $w_8.$ Without regularization, however, the optimized final values are not that robust to initializations and may be a bit extreme in terms of magnitudes. Therefore, we adopt the optimized final weights from $\alpha_3=0.001$ for further analysis.

We compare learned networks over time and provide some key statistics in Fig.~\ref{fig:key_statistics} of App.~\ref{app_sec:extended_plots_and_tables}. 
The local clustering coefficient is a measure of the degree to which nodes tend to cluster together. Closeness centrality measures how close a node is to all other nodes, calculated by finding the average shortest distance between a node and all other nodes in the network. Here, the distance is computed by taking inverse edge weights $w_e\rightarrow\frac{1}{w_e}$. More information about network summaries can be found in \citet{newman2018networks}. We observe evolution in average weighted degrees, average local clustering coefficient, and average closeness centrality values. The summary statistics indicate the presence of network dynamics even under our structural consistency assumption for consecutive time steps, revealing that structural stability allows gradual evolution rather than no dynamics at all.

%%%%%%%%%%%%%%%%%%%%%%%%%%%%%%%%
%%%%%%%%%%%%%%%%%%%%%%%%%%%%%%%%
\section{Node Similarity in Network Time Series} 
To detect strong bonds for potential structural driver analysis, we propose two notions of similarity between individuals based on node-wise close relations over time. The notion of relatedness is in principle user-defined; here we base it on whether two individuals are in the same ``community'', as explained below. One notion of similarity is based on how many times they are in close relationship over the period when they \emph{co-exist} (i.e., both are non-isolated). We denote this similarity notion as \emph{count similarity}. The other notion is based on the longest time interval during which they keep the close relations; this is the longest stretch of time the two individuals stay related, excluding the time steps when either or both are isolated. The resulting similarity notion can be used to understand the longest duration of close relations for each pair of entities within the network. We denote this notion as \emph{duration similarity}. Here, we only consider time steps where both nodes co-exist, and hence the number of time steps considered for each pair of nodes may differ. Based on the analysis in this section regarding distributions, for each observed similarity value, we can compute its p-value for the null hypothesis that nodes stay related independently and randomly across time. We conduct Bonferroni correction~\cite{vanderweele2019some} to identify significantly similar nodes.

\subsection{Theoretical Analysis of Similarity Notions}
\label{subsec:similarity_theorems}
We propose a novel definition of node similarity by testing the null hypothesis that the event of two nodes being closely related is drawn randomly and independently over time. To quantify the two notions of similarity, we carry out a theoretical analysis on sequences of independent Bernoulli trials with different success probabilities over time. The distribution of the number of successes of independent but not necessarily identically distributed Bernoulli random variables is called the Poisson-Binomial distribution. In general, there is no closed-form available for its probability mass function, but Thm.~\ref{thm:count_successes_distribution} in App.~\ref{app:proofs} provides a recursion formula for it. We use this recursion to assess the significance for the count similarity. Similarly, Thm.~\ref{thm:longest_consecutive_successes_distribution} in App.~\ref{app:proofs} gives a recursion formula for the longest success run in such a sequence, which we use to assess significance for duration similarity. 
%%%%%%%%%%%%%%%%%%%%%%%%%%%%%%%%
\subsection{Node Similarity via Community Identities}
A central concept here is the notion of \emph{close relations}. To yield independent samples, we propose to define two nodes to have a close relationship at a certain time step if and only if they share the same community identity, for a time step when they co-exist. For time steps when at least one node has no record, we disregard them in the computation.

Since we think of relatedness as being in the same community, we carry out community detection at every time step, yielding a sequence of partitions. For any two nodes $i$ and $j$, for each time step they co-exist, we record whether or not they are assigned to the same community, resulting in a sequence of entries taking value 0 (different communities) and 1 (same community), denoted as $\{B^t_{i,j}\}$. We then assess \emph{count similarity} via the number of times that $i$ and $j$ are in the same community, which is the number of 1's in this sequence; we also compute \emph{duration similarity} by the length of the longest shared path between them, which is the length of the longest run of consecutive 1's in this sequence. Applying Thm.~B.1 and B.2, we use the random variable $C_{i,j}^{T}$ to count the number of times nodes $i$ and $j$ belong to the same community for time stamps they co-exist, and use $D_{i,j}^{T}$ to denote the longest shared path throughout times they co-exist. We are then left to compute the success probabilities (i.e., probabilities of two nodes staying in the same community) over time, which is denoted as $\{p_{i,j}^t\}.$ Prop.~\ref{prop:prob_same_community_each_time} in App.~\ref{app:proofs} computes these probabilities to fully apply Thm.~\ref{thm:count_successes_distribution} and \ref{thm:longest_consecutive_successes_distribution}.
%%%%%%%%%%%%%%%%%%%%%%%%%%%%%%%%%%%%
\subsection{Application to Chimpanzee Networks}
For each yearly graph, we employ the popular Leiden algorithm~\cite{traag2019louvain} to construct communities for each time step. The Leiden algorithm, by default, detects communities without requiring user choice that need additional justification. The number of communities that the Leiden algorithm identifies matches the intuition of the chimpanzee experts involved in the study, and hence is deemed appropriate. In order to mitigate the effect of randomness inherent in the community detection algorithm, we run the Leiden algorithm 100 times for each network, and pick the partition with the largest value of {\em modularity}, a standard quality measure in community detection~\cite{newman2018networks}. 
Since focals are more often observed, merely thresholding the proximity counts would not reflect the data well. Instead, with the above theorems, for each observed similarity value, we compute its $p$-value for the null hypothesis that nodes are closely related (in our case, belonging to the same community) independently and randomly across time. In other words, the null hypothesis is: a pair of nodes stays in the same community at random at each time step. The alternative hypothesis is: a pair of nodes stays closer than random over time. We conduct the Bonferroni correction to select the most significantly similar nodes.

Fig.~\ref{fig:similarity_and_p_value_full} in App.~\ref{app_sec:extended_plots_and_tables} visualizes the similarity graphs (with similarity values as edge weights) and $p$-values for the learned yearly networks in chimpanzees. In order to further prove the concept of our graph combination merits, we compare the learned networks with two baselines. The first baseline (``unlearned") simply uses hardcoded, unlearned yearly networks by setting $w_j=0.1$ for $j>1$, $w_\text{add}=0.1,$ and $w_1=1.0$. The second baseline (``binary") treats all yearly networks as binary by setting all edge weights to one.

\begin{table}[htb]
\centering
  \centering
  \small
  \setlength{\tabcolsep}{3pt}
\begin{tabular}{clll}
\toprule
Method & Count Similarity & Duration Similarity \\ 
\midrule
\multirow{3}{*}{learned} & 
[\textbf{ri, hu}, ro, wn, ga],& [\textbf{ri, hu}, wn, ws, ro, ga], \\ 
& [cs, hi, mu], \textbf{[pe, ct]}, & [hi, mu, cs],\\ 
& \textbf{[rh, pi]}, [dx, mu], \textbf{[bt, pp]}
& \textbf{[pe, ct]}, \textbf{[rh, pi]}
\\ 
\hline
\multirow{3}{*}{unlearned} & 
[\textbf{ri, hu}, ro, ga, wn], \textbf{[pe, ct]}, 
 & [\textbf{ri, hu}, ro, ws, wn, ga], \\
&[ro, ri, ga, garbo], \textbf{[rh, pi]},
 & \textbf{[pe, ct]}, \textbf{[rh, pi]},\\
&[dx, mu], [ro, pi]&[dx, mu], [ro, pi]
\\
\hline
\multirow{2}{*}{binary} & 
[ws, hu], [mu, lo], [mu, sp], &[ws, hu], [mu, mg], \\
&[mu, cs], [mu, mg]&[mu, lo]\\
\bottomrule
\end{tabular}
\caption{Cliques detected by two notions of similarity on the chimpanzee networks for three network combination methods. Individuals are denoted by their codes. Known strong bonds are marked in bold.}
\label{tab:cliques}
\end{table}
For the full graphs as described by Eq.~\eqref{eq:adj_t}, Fig.~\ref{fig:thresholded_similarity} in App.~\ref{app_sec:extended_plots_and_tables} visualizes thresholded similarity graphs (keeping only significant entries) for learned, unlearned, and binary graphs, respectively, based on p-values and Bonferroni correction with a significance level of 0.05.

Based on the thresholded networks, we discover persistent cliques in Tab.~\ref{tab:cliques}. Cliques correspond to expectations based on qualitative observations over 10 years of observations on this population~\cite{mitani2009male}. Two maternal brothers appear in the same clique (ri and hu within a larger clique; pe and ct as a dyadic clique). Both the learned and the unlearned networks include those pairs in their significantly close dyads, but the binary networks fail to do so. These pairs are known to exhibit strong bonds (e.g., pe and ct were among the top grooming and proximity partners in a study with an independent data set collected during one year from 2014 to 2015, see \citet{sandel2020adolescent}). 

Overall, the results from the learned and unlearned networks are similar. In particular, we observe additional dyads, which, from long-term qualitative observations, were known to have a persistent “mentor-mentee” relationship. One pair includes an adult male (rh) and his biological father (pi); although chimpanzees do not appear to have kin recognition mechanisms for their biological fathers, there is evidence in this population that adolescent and young adult males preferentially groom and spend time in proximity to their biological fathers and other older “mentor” figures~\cite{sandel2020adolescent}. The other pair, which is only captured by the learned networks, involves an adult male, bt, who “adopted” pp, a younger male as a juvenile, and the two remained close in adulthood. Comparing the shapes of the enduring relationships from Fig.~\ref{fig:thresholded_similarity} in App.~\ref{app_sec:extended_plots_and_tables}, learned graphs typically produce fewer ``tails" and more actual ``cliques", resulting in more stable and cohesive bonds among individuals. In line with \citet{SexualSegregationCliquesandSocialPowerinSquirrelMonkeySaimiriGroups}, we conjecture that cliques may have an advantage, and perhaps individuals in these cliques may be key drivers of the social structure in the chimpanzee population.
%%%%%%%%%%%%%%%%%%%%%%%%%
\section{Conclusion and Future Work}
\label{sec:conclusion}
This paper develops a network representation and a subsequent analysis of a novel data set of chimpanzee interactions. To this purpose, it provides a novel optimization approach to combine proximity networks into a single network based on hierarchies of proximity levels. It also gives a principled way to identify long-term related nodes in network time series. 

We anticipate ProxFuse may also be useful for the analysis of similar human or animal interaction data sets.  For the similarity analysis, we plan to make the method more robust to randomness in the community assignments. We explore here potential structural drivers via a novel lens of significant long-term relationships. Further analysis, e.g., in terms of influential individuals or how these significant relationships drive structural changes, is left as future work. 
%%%%%%%%%%%%%%%%%%%%%%%%%
\section*{Acknowledgement}
The chimpanzee data set was collected with the approval of the Uganda Wildlife Authority, Uganda National Council for Science and Technology, and the Makerere University Biological Field Station. We thank members of the Ngogo Chimpanzee Project who provided support in the field, especially David Watts, Kevin Langergraber, Sam Angedakin and the late Jerry Lwanga. For support with data organization, we thank Veronika Städele. Research at Ngogo has been funded by: Arizona State University; Institute of Human Origins; Keo Films; L.S.B. Leakey Foundation; Max Planck Society; National Science Foundation (BCS-9253590, IOS-0516644, BSC-0850951, BCS-1613393, BSC-1850328, BCS- 1540259, BCS-0215622); National Institute on Aging (R01-AG049395); National Geographic Society; Wenner-Gren Foundation (including Gr. 9957); Silverback Films; Underdog Films; Wildstar Films; University of Michigan; and Yale University.

In addition, Gesine Reinert is funded in part by EPSRC grants EP/T018445/1, EP/V056883/1, EP/Y028872/1,  and EP/X002195/1.
%%%%%%%%%%%%%%%%%%%%%%%%%
\newpage
\bibliography{main_arxiv}
\bibliographystyle{arxiv_style}

%%%%%%%%%%%%%%%%%%%%%%%%%%%%%%%%%%%%%%%%%%%%%%%%%%%%%%%%%%%%%%%%%%%%%%%%%%%%%%%
%%%%%%%%%%%%%%%%%%%%%%%%%%%%%%%%%%%%%%%%%%%%%%%%%%%%%%%%%%%%%%%%%%%%%%%%%%%%%%%
% APPENDIX
%%%%%%%%%%%%%%%%%%%%%%%%%%%%%%%%%%%%%%%%%%%%%%%%%%%%%%%%%%%%%%%%%%%%%%%%%%%%%%%
%%%%%%%%%%%%%%%%%%%%%%%%%%%%%%%%%%%%%%%%%%%%%%%%%%%%%%%%%%%%%%%%%%%%%%%%%%%%%%%
\newpage
\appendix
\onecolumn
\section{Proximity Types}
\label{app_sec:proximity_types}
\begin{table*}[h]
    \centering
    \begin{tabular}{c|c|c|c|c}
    \toprule
        Individual 1/Individual 2 & Party&Prox5&Prox2 (including grooming)&Focal  \\
        \midrule
        Party &Type 1&Type 2&Type 3&Type 4\\
        \hline
        Prox5 &Type 2&Type 5&Type 6&Type 7\\ 
        \hline
        Prox2 (including grooming) &Type 3&Type 6&Type 8&Type 9 or Type 10 (if grooming)\\
        \hline
        Focal &Type 4&Type 7&Type 9 or Type 10 (if grooming)&Impossible\\ 
        \bottomrule
    \end{tabular}
    \caption{Relationships to the focal for two different individuals, and the resulting type of interactions for this pair of individuals.}
    \label{tab:proximity_types}
\end{table*}
We use the following proximity types based on the level of proximity two individuals have at a certain recorded time (examples are from Fig.~\ref{fig:proximity_types} in the main text), with an illustration table in Table~\ref{tab:proximity_types}:
\begin{itemize}
    \item Type 1: Two individuals are both in the ``party" to a focal male (subject of observation), but not observed in ``prox2" or ``prox5". E.g., J \& K. In particular, two individuals are roughly within 200m of each other, although they may be much closer than that.
    \item Type 2: One individual is in ``prox5" to the focal, and the other in ``party", but not observed in ``prox2" or ``prox5". E.g., H \& K. Thus, they must be within roughly 105m of each other.
    \item Type 3:
    One individual is in ``prox2" to the focal, and the other individual is in ``party". E.g., A \& I. Thus, they must be within roughly 102m of each other.
    \item Type 4: Subject of observation (focal) + an individual in its ``party", but not observed in 2m/5m proximity or grooming. E.g., F \& K. Thus, they must be within roughly 100m of each other.
    \item Type 5: Two individuals are both in ``prox5" to the focal. E.g., G \& H. Thus, they must be within roughly 10m of each other.
    \item Type 6: One individual is in ``prox2" to the focal and the other is in ``prox5". E.g., A \& H. Thus, they must be within roughly 7m of each other.
    \item Type 7: Subject of observation (focal) + an individual in its ``prox5". E.g., F \& E. Thus, they must be within roughly 2-5m of each other.
    \item Type 8: Two individuals are both in ``prox2" to the focal. E.g., A \& B. Thus, they must be within roughly 4m of each other.
    \item Type 9: Subject of observation (focal) + an individual in its ``prox2" but not grooming. E.g., F \& A. Thus, they must be within roughly 2m of each other.
    \item Type 10 (within 2m but semantically closer than type 9): Subject of observation (focal) + a grooming individual in its ``prox2". E.g., F \& B. Thus, as for Type 9, they must be within roughly 2m of each other; for Type 10, they must be grooming in addition to being within roughly 2m apart.
\end{itemize}

%%%%%%%%%%%%%%%%%%%%%%%%%%%%%%%%%%%%
\section{Theorems and Proofs}\label{app:proofs}
In this section, we provide theoretical results for the two novel notions of node similarity, with detailed proofs. All probability formulas have also been empirically validated to have probabilities add up to one for all cases we consider.
\subsection{Theorem and Proof of Count Similarity}
Here, we provide a theorem and the proof relating to count similarity. We note that the distribution of $C_t$ is also called a {\it Poisson-binomial distribution.} In general, there is no closed form available for its probability mass function.

\begin{theorem}
\label{thm:count_successes_distribution}
[Useful for Count Similarity] For a 
series of independent Bernoulli trials $\{B_t\}$ of length $T\;(T\geq 2)$, $t\in \{1, \dots, T\},$ with success rate $p_t$ for $B_t,$ denote $C_t \in \{0, \dots, t\}$ as the total number of successes in time steps $\{ 1, \dots, t\}$. 
For $t\in\{1, \dots, T\}$, the probability distribution of $C_t$ satisfies
\begin{align}
\mathbb{P}(C_t=0) &=\prod_{s=1}^t(1-p_s), \;\;
 \mathbb{P}(C_t=t) =\prod_{s=1}^t p_s,
 \end{align}
 and for $L\in\{1, \dots, t-1\}$ if $t\geq 2,$
 \begin{align}
 \mathbb{P}(C_t=L) &=p_t\cdot\mathbb{P}(C_{t-1}=L-1) + (1-p_t)\cdot \mathbb{P}(C_{t-1}=L);
    \end{align}
    for $ L\geq t+1,$
    \begin{align}
    \mathbb{P}(C_t=L)&=0.
\end{align}
\end{theorem}

\begin{proof}

We first prove the behavior at the boundaries.

Since the total number of successes is bounded by the number of trials, we always have $C_t\leq t.$ Therefore, $\mathbb{P}(C_t=L)=0$ whenever $ L\geq t+1.$ The only chance of having pure successes is to never fail the trial, and hence $\mathbb{P}(C_t=t) =\prod_{s=1}^t p_s, \;\; t\in\{1, \dots, T\},$ due to the independence between the Bernoulli random variables. Similarly, the only chance of having zero successes is to fail the trial every time, and hence $\mathbb{P}(C_t=0) =\prod_{s=1}^t(1-p_s), \;\; t\in\{1, \dots, T\}$.

In order to compute $\mathbb{P}(C_t=L)$ for $t\in\{2, \dots, T\},\;L\in\{1, \dots, t-1\},$ we first analyze the role of this time step $t$. There are two possibilities for time step $t$: either it is a success, or it is a ``failure" point. Note that we concentrate on the time series from the start (time step $1$) till time step $t,$ and we impose no constraints on future time steps. For the first situation, we require that we have $L-1$ successes before $t;$ while for the second situation, we require $L$ successes at time $t-1.$ Therefore, 
\begin{align*}
    \mathbb{P}(C_t=L) &=p_t\cdot\mathbb{P}(C_{t-1}=L-1)+ (1-p_t)\cdot \mathbb{P}(C_{t-1}=L), \\
    & \qquad \qquad \;\;t\in\{2, \dots, T\},\;L\in\{1, \dots, t-1\}.
\end{align*}
Combing the above, we have that 
\begin{align*}
\begin{split}
    \mathbb{P}(C_t=t) &=\prod_{s=1}^t p_s, \;\; t\in\{1, \dots, T\},\\
    \mathbb{P}(C_t=0) &=\prod_{s=1}^t(1-p_s), \;\; t\in\{1, \dots, T\},\\
    \mathbb{P}(C_t=L) &=p_t\cdot\mathbb{P}(C_{t-1}=L-1)+ (1-p_t)\cdot \mathbb{P}(C_{t-1}=L),\\
    & \qquad \qquad  \;\;t\in\{2, \dots, T\},\;L\in\{1, \dots, t-1\},\\
    \mathbb{P}(C_t=L)&=0,\;\;t\in\{1, \dots, T-1\}, \; L\in\{t+1, \dots, T\}.
\end{split}
\end{align*}
\end{proof}
%%%%%%%%%%%%%%%%%%%%%%%%%%%
\subsection{Theorem and Proof for Duration Similarity}
The next theoretical result gives a means for assessing significant duration similarity.
\begin{theorem}
\label{thm:longest_consecutive_successes_distribution}
[Useful for Duration Similarity] For a time series of independent Bernoulli trials $\{B_t\}$ of length $T\;(T\geq 2)$,  $t\in\{1, \dots, T\},$ with success rate $p_t$ for $B_t,$ denote $D_t \in \{0, \dots, t\}$ as the longest consecutive successes from the start until time $t.$  The probability distribution satisfies for $t\in\{1, \dots, T\}$,
\begin{align}
    \mathbb{P}(D_t=L)&=0,\;
    L\in\{t+1, \dots, T\},\\
    \mathbb{P}(D_t=t) &=\prod_{s=1}^t p_s, \quad \quad 
    \mathbb{P}(D_t=0) =\prod_{s=1}^t(1-p_s);
    \end{align}
    for $t\in\{2, \dots, T\},$
    \begin{align}
    \mathbb{P}(D_t=t-1)
    &=(1-p_t)\cdot\prod_{s=1}^{t-1} p_s +(1-p_1)\cdot \prod_{s=2}^{t} p_s,
    \end{align}
    for $t\in\{3, \dots, T\}$ if $T\geq 3$,
    \begin{align}
    \mathbb{P}(D_t=1) &=\left[\sum_{l=0}^{1}\mathbb{P}(D_{t-2}=l)\right]\cdot(1-p_{t-1})\cdot p_t \nonumber\\
    &+ (1-p_t)\cdot \mathbb{P}(D_{t-1}=1);
    \end{align}
    and for $t\in\{4, \dots, T\}$ and $L\in\{2, \dots, t-2\}$ if $T\geq 4$,
    \begin{align}
    \mathbb{P}(D_t=L) =& \left[\sum_{l=0}^{L}\mathbb{P}(D_{t-L-1}=l)\right]\cdot
   (1-p_{t-L})\cdot\left(\prod_{s=t-L+1}^t p_s\right) \nonumber\\
    &+ \sum_{s=t-L+1}^{t-1}(1-p_s)
    \left(\prod_{k=s+1}^t p_k\right)\cdot
    \mathbb{P}(D_{s-1}=L)\\
    &+ (1-p_t)
    \mathbb{P}(D_{t-1}=L).
\end{align}
\end{theorem}

To prove Thm.~\ref{thm:longest_consecutive_successes_distribution} we first show the following result. 

\begin{proposition}\label{prop:independence_M_B}
For a time series of independent Bernoulli trials $\{B_t\}$ of length $T\;(T\geq 2)$, $t\in\{1, \dots, T\},$ with success rate $p_t$ for $B_t,$ denote $D_t$ as the longest consecutive successes from the start until time $t,$ then $D_t$ takes values from $\{0, \dots, t\}.$  Further, $D_t$ and $B_s$ are independent if $s>t.$
\end{proposition}
\begin{proof}
By definition, the number of successes takes integer values from 0 to the number of trials, i.e., $D_t$ takes values from $\{0, \dots, t\}.$ Further, $D_t$ is only dependent on $B_1, \dots, B_t,$ and $B_s$ is independent of $B_k$ for any $k\neq s.$ Given $s>t,$ we have that $D_t$ and $B_s$ are independent.
\end{proof}

Now we present and prove Thm.~\ref{thm:longest_consecutive_successes_distribution}.

\begin{proof}
We first prove the behavior at the boundaries.

Since the total number of successes is bounded by the number of trials, we always have $D_t\leq t.$ Therefore, $\mathbb{P}(D_t=L)=0$ whenever $L\geq t+1.$ The only chance of having pure successes is to never fail the trial, and hence $$\mathbb{P}(D_t=t) =\prod_{s=1}^t p_s, \;\; t\in\{1, \dots, T\}$$ due to the independence between the Bernoulli random variables. Similarly, the only chance of having zero successes is to fail the trial every time, and hence $$\mathbb{P}(D_t=0) =\prod_{s=1}^t(1-p_s), \;\; t\in\{1, \dots, T\}.$$ 
For $\mathbb{P}(D_t=t-1)$ with $t\in\{2, \dots, T\},$ we require exactly one failure at either the very end or the very beginning. Therefore, 
\begin{align*}\mathbb{P}(D_t=t-1)&=(1-p_t)\cdot\left(\prod_{s=1}^{t-1} p_s\right)+(1-p_1)\cdot\left(\prod_{s=2}^{t} p_s\right),\\& t\in\{2, \dots, T\}.\end{align*}
In order to compute $\mathbb{P}(D_t=L)$ for $t\in\{3, \dots, T\},\;L\in\{1, \dots, t-2\}$ if $T\geq 3,$ we first analyze the role of this time step $t$. There are two possibilities for time step $t$: either it is an endpoint for a chain with $L$ consecutive successes, or it is not such a point. Note that we concentrate on the time series from the start (time step $1$) till time step $t,$ and we impose no constraints on future time steps.

For the first situation, we require the point before this chain to be a ``failure" point, and that the longest consecutive success length before this ``failure" point is no more than $L$ (since $L$ is achieved by the chain already and we need to be consistent with the definition of the largest length). Mathematically, we require that 
\begin{itemize}
    \item $\prod_{s=t-L+1}^tB_s=1$ for the definition of the chain containing $L$ consecutive successes until $t$; 
    \item $B_{t-L}=0$ for the definition of the ``failure" point;
    \item $D_{t-L-1}\leq L$ as $L$ is defined to be the largest length until $t$ and that this largest length could be achieved more than once (and hence we take $\leq L$ instead of $< L$). Note that the definition of $D_{t-L-1}$ is valid as we have $t-L-1\geq t-(t-2) - 1 \geq 1.$ 
\end{itemize}

For the second situation, since time step $t$ is not an endpoint of a chain with $L$ consecutive successes ending at $t$, and that $L\leq t-2,$ there must be at least one ``failure" point within the time steps $t-L+1, \dots, t.$ Denote the last ``failure" time step before or at time step $t$ as $s\in\{t-L+1, \dots, t\},$ then $B_s=0.$ If $s+1\leq t$, then by definition, all points after $s$ should be successes, i.e., $\prod_{k=s+1}^tB_k=1$. In addition, we require that the largest length is $L$ before this ``failure" point, i.e., $D_{s-1}=L$. Here, the definition of $D_{s-1}$ is valid as $s-1\geq (t-L+1)-1\geq t-L\geq 2\geq 1.$ To summarize, for the second situation, the mathematical requirements are 
\begin{itemize}
    \item $B_s=0$ for some $s\in\{t-L+1, \dots, t\};$
    \item $\prod_{k=s+1}^tB_k=1$ if $s+1\leq t;$ 
    \item $D_{s-1}=L$.
\end{itemize}

The second point implicitly assumes $(t-L+1)+1\leq t,$ i.e., $L\geq 2;$ then since $L\leq t-2,$ this further implies that $t\geq 4,$ and hence it is only possible when $T\geq 4.$

Therefore, we arrive at the following recurrence relations. If $T\geq 4,$ for $t\in\{4, \dots, T\},\;L\in\{2, \dots, t-2\},$ with $\mathbf{1}(\cdot)$ being an indicator function:
\begin{align}
\begin{split}
\label{eq:D_tL_recurrence}
   & \mathbb{P}(D_t=L) \\ & =  \mathbb{P}\left[\mathbf{1}(D_{t-L-1}\leq L) \bigcap \mathbf{1}(B_{t-L}=0)\bigcap \mathbf{1}(B_s=1\;\forall s=t-L+1, \dots, t)\right]\\
    &+ \sum_{s=t-L+1}^{t-1}\mathbb{P}\left[\mathbf{1}\left(B_s=0\right)\bigcap \mathbf{1}\left(B_k=1\;\forall \;k=s+1, \dots, t\right)\bigcap \mathbf{1}\left(D_{s-1}=L\right)\right]\\
    &+ \mathbb{P}\left[\mathbf{1}\left(B_t=0\right)\bigcap \mathbf{1}\left(D_{t-1}=L\right)\right]\\
    &= \left[\sum_{l=0}^{L}\mathbb{P}(D_{t-L-1}=l)\right]\cdot\mathbb{P}(B_{t-L}=0)\cdot\left[\prod_{s=t-L+1}^t \mathbb{P}(B_s=1)\right]\\
    &+ \sum_{s=t-L+1}^{t-1}\mathbb{P}(B_s=0)\cdot\left[\prod_{k=s+1}^t \mathbb{P}(B_k=1)\right]\cdot\mathbb{P}(D_{s-1}=L) + \mathbb{P}(B_t=0)\cdot \mathbb{P}(D_{t-1}=L)\\
    &= \left[\sum_{l=0}^{L}\mathbb{P}(D_{t-L-1}=l)\right]\cdot(1-p_{t-L})\cdot\left(\prod_{s=t-L+1}^t p_s\right)\\
    &+ \sum_{s=t-L+1}^{t-1}(1-p_s)\cdot\left(\prod_{k=s+1}^t p_k\right)\cdot\mathbb{P}(D_{s-1}=L) + (1-p_t)\cdot \mathbb{P}(D_{t-1}=L).
\end{split}
\end{align}
Note that the products could be taken as the $B_m$ terms have their indices $m$ greater than those from the $D_l$ terms, i.e., $m>l$, given Prop.~\ref{prop:independence_M_B}. 
Specifically, in the first term in the summation of Eq.~(\ref{eq:D_tL_recurrence}), $l=t-L-1<m$ for $m\in\{t-L, \dots, t\};$ in the second term of summation, $l=s-1<m$ for $m\in\{s, s+1, \dots, t\};$ for the last term in the summation, $l=t-1<t=m.$

For $L=1$ and $t\in\{3, \dots, T\}$ if $T\geq 3,$ we have, similarly,
\begin{align}
\begin{split}
    &\mathbb{P}(D_t=1) \\
    &= \mathbb{P}\left[\mathbf{1}(D_{t-2}\leq 1) \bigcap \mathbf{1}(B_{t-1}=0)\bigcap \mathbf{1}(B_t=1)\right] + \mathbb{P}\left[\mathbf{1}\left(B_t=0\right)\bigcap \mathbf{1}\left(D_{t-1}=1\right)\right]\\
    &= \left[\sum_{l=0}^{1}\mathbb{P}(D_{t-2}=l)\right]\cdot\mathbb{P}(B_{t-1}=0)\cdot\mathbb{P}(B_t=1) + \mathbb{P}(B_t=0)\cdot \mathbb{P}(D_{t-1}=1)\\
    &= \left[\sum_{l=0}^{L}\mathbb{P}(D_{t-2}=l)\right]\cdot(1-p_{t-1})\cdot p_t + (1-p_t)\cdot \mathbb{P}(D_{t-1}=1).
\end{split}
\end{align}

Combining the above, we have that 
\begin{align*}
     \mathbb{P}(D_t=L)&=0,\;\;t\in\{1, \dots, T\}, \; L\geq t+1,
\end{align*}
as well as 
\begin{align*}
\begin{split}
    \mathbb{P}(D_t=t) &=\prod_{s=1}^t p_s, \;\; t\in\{1, \dots, T\},\\
    \mathbb{P}(D_t=0) &=\prod_{s=1}^t(1-p_s), \;\; t\in\{1, \dots, T\},\\
    \mathbb{P}(D_t=t-1)
    &=(1-p_t)\cdot\left(\prod_{s=1}^{t-1} p_s\right)+(1-p_1)\cdot\left(\prod_{s=2}^{t} p_s\right),\;\;t\in\{2, \dots, T\},\\
    \mathbb{P}(D_t=1) &=\left[\sum_{l=0}^{1}\mathbb{P}(D_{t-2}=l)\right]\cdot(1-p_{t-1})\cdot p_t\\
    &+ (1-p_t)\cdot \mathbb{P}(D_{t-1}=1), \;\;t\in\{3, \dots, T\}\text{ if }T\geq 3,\\
    \mathbb{P}(D_t=L) &=\left[\sum_{l=0}^{L}\mathbb{P}(D_{t-L-1}=l)\right]\cdot(1-p_{t-L})\cdot\left(\prod_{s=t-L+1}^t p_s\right)\\
    &+ \sum_{s=t-L+1}^{t-1}(1-p_s)\cdot\left(\prod_{k=s+1}^t p_k\right)\cdot\mathbb{P}(D_{s-1}=L)\\
    &+ (1-p_t)\cdot \mathbb{P}(D_{t-1}=L), \;\;t\in\{4, \dots, T\}\text{ if }T\geq 4,\;L\in\{2, \dots, t-2\}.
\end{split}
\end{align*}
\end{proof}
%%%%%%%%%%%%%%%%%%%%%%%%%%%%%%%
\subsection{Proposition and Proof for Same-Community Probability}
\begin{proposition}\label{prop:prob_same_community_each_time}
Suppose for a time step $t,$ both nodes $i$ and $j$ exist in the network containing $n_t$ nodes and $K_t$ communities $\mathcal{C}_1^{t}, \dots, \mathcal{C}_{K_t}^{t}.$ Suppose all nodes have the same i.i.d. community assignment distribution, then the distribution of $i$ and $j$ being in the same community at time step $t$ is a Bernoulli random variable $B_{i,j}^t$ with success probability 
\begin{equation*}
  p_{i,j}^t=\sum_{k=1}^{K_t}\frac{\left\lvert\mathcal{C}_k^{t}\right\rvert\left(\left\lvert\mathcal{C}_k^{t}\right\rvert-1\right)}{n_t(n_t-1)}.
\end{equation*}
\end{proposition}

\begin{proof}
As all nodes have the same i.i.d. community assignment distribution, the probability for nodes $i$ and $j$ being both in community $\mathcal{C}_k^{t}$ for some $k\in\{1, \dots, K_t\}$ is
\begin{equation*}
\frac{\binom{\left\lvert\mathcal{C}_k^{t}\right\rvert}{2}}{\binom{n_t}{2}}=\frac{\left\lvert\mathcal{C}_k^{t}\right\rvert\left(\left\lvert\mathcal{C}_k^{t}\right\rvert-1\right)/2}{n_t(n_t-1)/2}=\frac{\left\lvert\mathcal{C}_k^{t}\right\rvert\left(\left\lvert\mathcal{C}_k^{t}\right\rvert-1\right)}{n_t(n_t-1)}.
\end{equation*}
Taking into account all communities at time step $t,$ we have 
\begin{equation}
    p_{i,j}^t=\sum_{k=1}^{K_t}\frac{\left\lvert\mathcal{C}_k^{t}\right\rvert\left(\left\lvert\mathcal{C}_k^{t}\right\rvert-1\right)}{n_t(n_t-1)}.
\end{equation}
This completes the proof.
\end{proof}
%%%%%%%%%%%%%%%%%%%%%%%%%%%%%%%%%%%%%%%%%%%%%%%%%%%%%%%%%%%%%%%%%%%%%%%%%%%%%%%
\section{Implementation Details}
\label{app_sec:implementation}
\subsection{Experimental Setup}
\label{app_sec:experiment_setup}
Experiments were conducted on two compute nodes, each with 8 Nvidia Tesla T4, 96 Intel Xeon Platinum 8259CL CPUs @ 2.50GHz and $378$GB RAM. We run at most 5000 epochs using the Adam optimizer~\cite{kingma2014adam} with a learning rate of $0.1$ and an early stopping parameter to be 3000 epochs. Anonymized codes are provided at \url{https://anonymous.4open.science/r/ProxFuse}. Experimental results are averaged over three runs on different random seeds. Since we have multiple combinations of initial parameters, the total number of runs for each synthetic data set is $3\times(6+H_\text{syn})=3\times(6+5)=33$ and $3\times(6+H)=3\times(6+10)=48$ for the chimpanzee data set, given a fixed set of $\alpha_1, \alpha_2,$ and $\alpha_3$ values. Note also that indices from Python start from 0, so the first time step in the code is marked by $t=0$ instead of $t=1$ as in the manuscript.
%%%%%%%%%%%%%%%%%%%%%%%%
\subsection{Synthetic Network Model}
\label{app_sec:synthetic_generation}
With the aim of constructing evolving individual networks, $\mathcal{G}^{(\text{raw}, h,t)}$ and $\mathcal{G}^{(\text{add}, h,t)}$, but stable cumulative network $\mathcal{G}^{(t)}=(\mathcal{N}^{(t)},\mathcal{E}^{(t)})$ over time, the synthetic models are constructed as follows. Let $n$ be the number of nodes in total, $T$ be the number of time steps, $H$ be the number of hierarchies (e.g., different proximity levels), and $\{p_h\}_{h=1}^H$ edge probabilities for each layer of the multiplex network. 

\textbf{Initialization.} For the first time step, $t=1,$  we assume that $\mathcal{G}^{(\text{raw}, h,t)}$, the raw single-type network for type $h$ at each time step $t,$ is generated independently as a Bernoulli random graph with $n$ nodes and edge probability $p_h$. The edge weights are sampled randomly from integers $\{1, \dots, H\}$. We use $\mathbf{A}^{(\text{raw}, h,t)}$ to denote its adjacency matrix. We then generate $\mathbf{A}^{(\text{add}, h,t)}$ by a Hadamard product of the binary version of $\mathbf{A}^{(\text{raw}, h,t)}$ with another Bernoulli random graph with $n$ nodes and edge probability $p_\text{add}$ (but the edge weights are again sampled randomly from integers $\{1, \dots, H\}$). Suppose networks are combined based on Eq.~(\ref{eq:adj_t}) in the main text. We then obtain
$$
\mathbf{A}^{(1)} = \sum_{h=1}^{H} \left( w_h \cdot \mathbf{A}^{(\text{raw}, h, 1)} + w_{\text{add}} \cdot \mathbf{A}^{(\text{add}, h, 1)} \right).$$
This fixed network $\mathbf{A}^{(1)}$ is used to initialize the process. 

At each subsequent time step $t > 1$, $\mathbf{A}^{(t)}$ is kept fixed and equal to $\mathbf{A}^{(1)}$, and the raw and ancillary networks are constructed through a randomized decomposition as follows.

\textbf{Edge assignment.} At time step $t > 1$, the edges of $\mathcal{G}^{(t)}$ are shuffled and redistributed among $H$ hierarchies based on normalized probabilities $\{p_h\}_{h=1}^H$. The number of edges assigned to each hierarchy $h$ is sampled from a multinomial distribution:
$$
\left\lvert\mathcal{E}^{h,t}\right\rvert \sim \text{Multinomial}\left(\left\lvert\mathcal{E}^{(1)}\right\rvert, \frac{p_h}{\sum_{h=1}^{H} p_h}\right),$$ where $\left\lvert\mathcal{E}^{(1)}\right\rvert$ denotes the total number of edges in the initial combined network $\mathcal{G}^{(1)}$ described in $\mathbf{A}^{(1)}.$ We then randomly assign the edges to each hierarchy based on $\left\lvert\mathcal{E}^{(h,t)}\right\rvert$, and construct subgraphs $\mathcal{G}^{(h,t)}=(\mathcal{N}^{(t)}, \mathcal{E}^{(h,t)})$, which share the same node set as $\mathcal{G}^{(t)}$, such that $\mathcal{E}^{(t)}=\cup_h\mathcal{E}^{(h,t)}$ and $\cap_h\mathcal{E}^{(h,t)}=\emptyset.$
The hierarchy-level combined adjacency matrix $\mathbf{A}^{(h,t)}$ is computed by normalizing the corresponding subgraph adjacency matrix:
$$
\mathbf{A}^{(h,t)}_{i,j} = \frac{\mathbf{A}^{(t)}_{i,j}}{W_h}= \frac{\mathbf{A}^{(t)}_{i,j}}{\sum_{k=1}^{h} w_k}
$$ for $(i,j)\in\mathcal{E}^{(h,t)}$ and $\mathbf{A}^{(h,t)}_{i,j} =0$ for $(i,j)\notin\mathcal{E}^{(h,t)}$ . 
This ensures that the weighted contributions of the network in hierarchy $h$ align with the structure of $\mathbf{A}^{(t)}$.

\textbf{Network construction.}
Recall that $\mathbf{A}^{(h,t)}=\mathbf{A}^{(\text{raw}, h, t)} + w_\text{add}\mathbf{A}^{(\text{add}, h, t)},$ and that the ancillary networks $\mathcal{G}^{(\text{add}, h, t)}$ are generated as subgraphs of the raw networks, $\mathcal{G}^{(\text{raw}, h, t)}$. First, we sample edges for $\mathcal{G}^{(\text{add}, h, t)}$ probabilistically based on $ p_{\text{add}}$, from $\mathcal{G}^{(h, t)}$. We then generate a temporary ancillary network with edge weights randomly sampled from $\{1, \dots, H\}.$ We denote the adjacency matrix of this temporary ancillary network by $\mathbf{A}^{(\text{temp-add}, h, t)}$. If $w_\text{add}=0,$ then we set $\mathbf{A}^{(\text{add}, h, t)}=\mathbf{A}^{(\text{temp-add}, h, t)}$. Otherwise, the edge weights of $\mathcal{G}^{(\text{add}, h, t)}$  are constrained to ensure nonnegativity and consistency:
$$
\mathbf{A}^{(\text{add}, h, t)}=\max\left(0, \min\left(\mathbf{A}^{(\text{temp-add}, h, t)}, \frac{\mathbf{A}^{(h,t)}}{w_{\text{add}}} - \epsilon\right)\right),
$$ by taking elementwise minimum and maximum, where $\epsilon$ is a small constant to ensure that $\mathcal{G}^{(\text{add}, h, t)}$ is a subgraph of $\mathcal{G}^{(\text{raw}, h, t)}$. We then construct $\mathcal{G}^{(\text{raw}, h, t)}$ by
$\mathbf{A}^{(\text{raw}, h, t)}= \mathbf{A}^{(h,t)}- w_\text{add}\mathbf{A}^{(\text{add}, h, t)}.$

At each time step, the nodes with nonzero degrees in $\mathcal{G}^{(t)}$ are identified as participating/existing nodes. 

This method ensures a controlled, randomized decomposition of $\mathbf{A}^{(t)}$ into hierarchical raw and ancillary networks, while preserving overall structural integrity and allowing for hierarchical variability.
%%%%%%%%%%%%%%%%%%%%%%%%%%%%%%%%%%%%%%%%%%%%%%%%%%%%%%%%%%%%%%%%%%%%%%%%%%%%%%%
\section{Extended Plots and Tables}
\label{app_sec:extended_plots_and_tables}

\begin{figure}[htb!]
\centering
    \begin{subfigure}[ht]{0.19\linewidth}
      \centering
      \includegraphics[width=\linewidth]{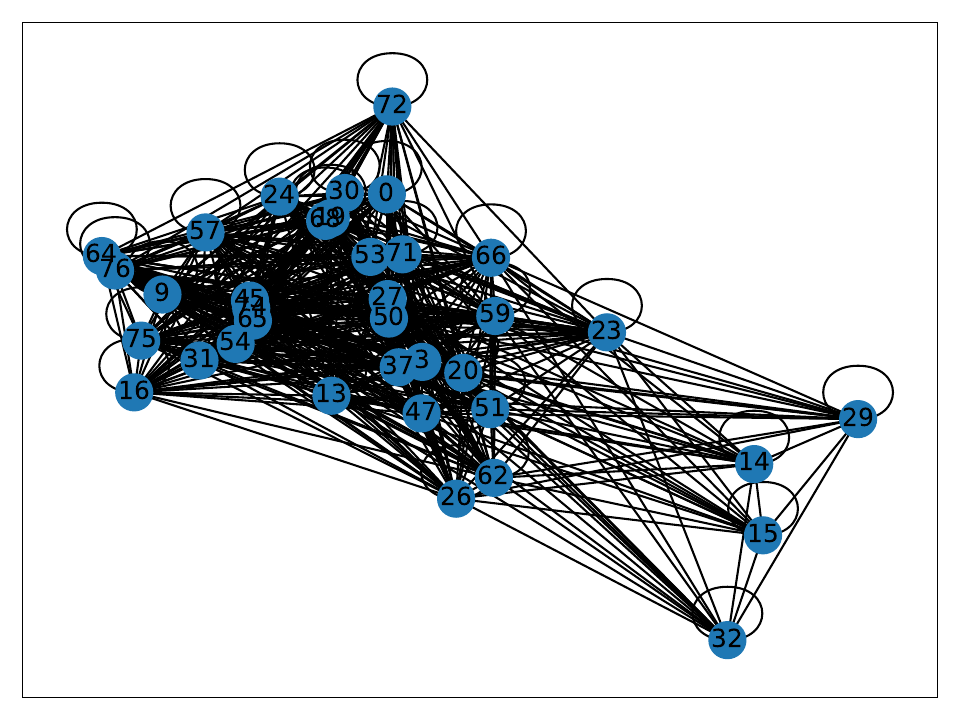}
      \subcaption{Type 1.}
    \end{subfigure}
    \begin{subfigure}[ht]{0.19\linewidth}
      \centering
      \includegraphics[width=\linewidth]{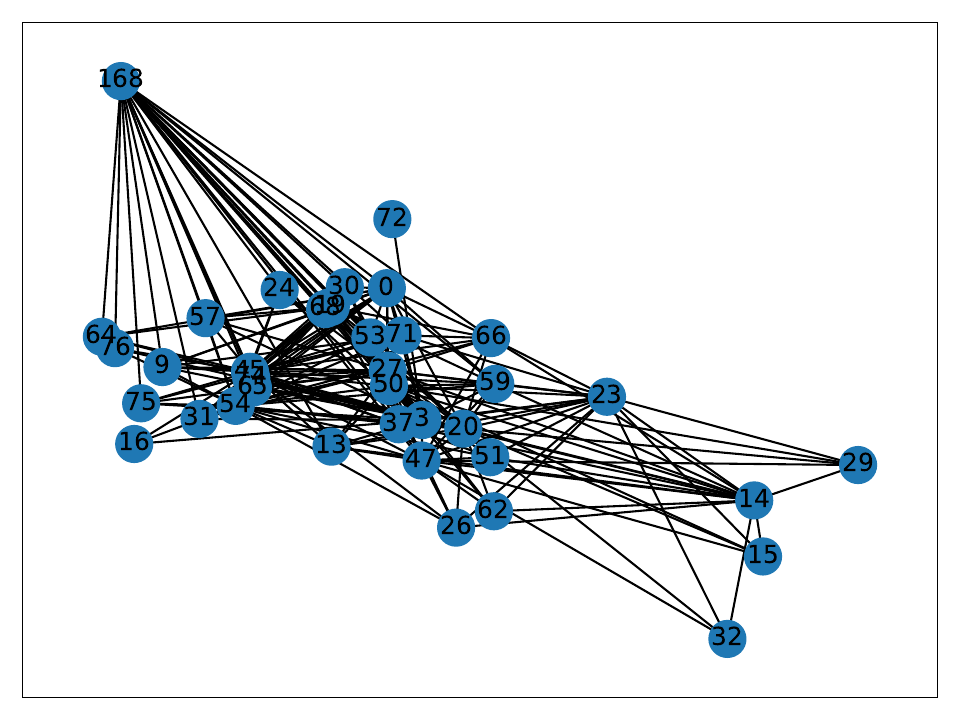}
      \subcaption{Type 2.}
    \end{subfigure}
    \begin{subfigure}[ht]{0.19\linewidth}
      \centering
      \includegraphics[width=\linewidth]{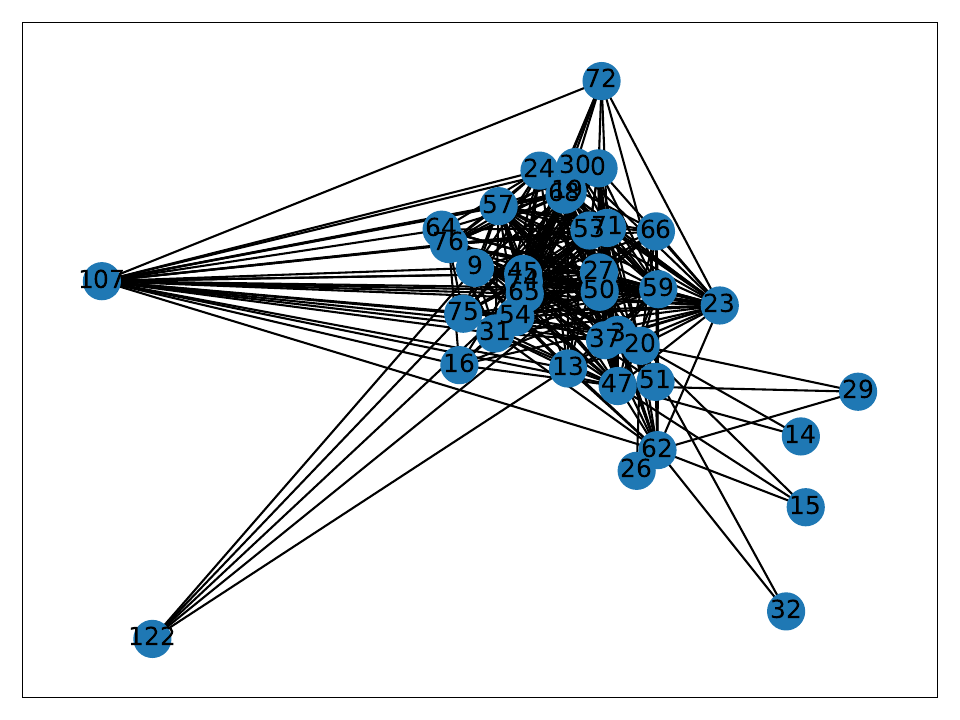}
      \subcaption{Type 3.}
    \end{subfigure}
    \begin{subfigure}[ht]{0.19\linewidth}
      \centering
      \includegraphics[width=\linewidth]{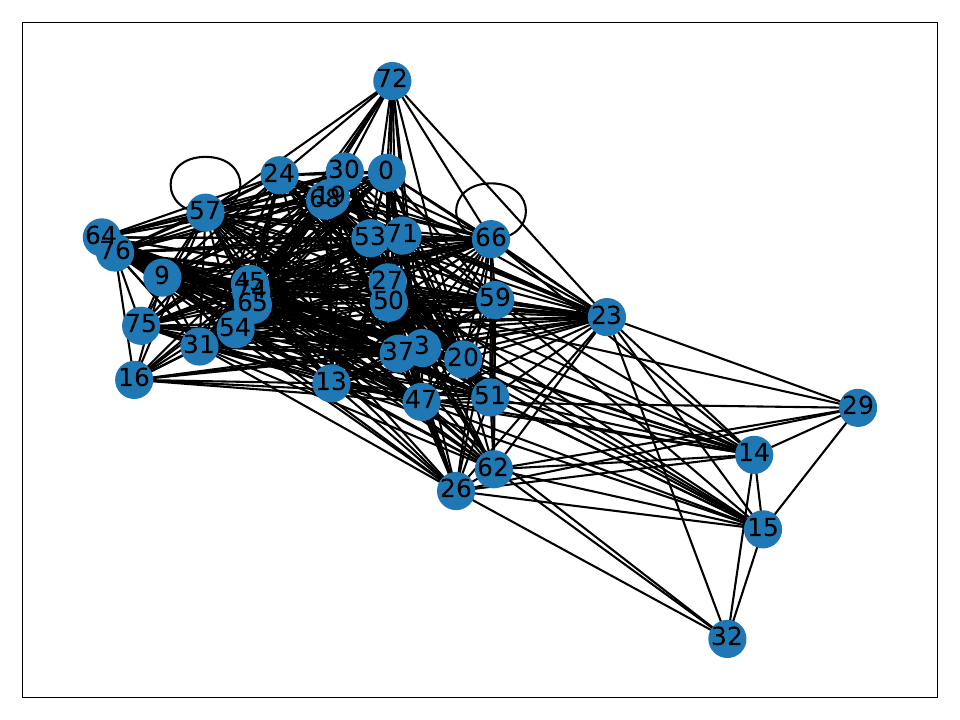}
      \subcaption{Type 4.}
    \end{subfigure}
    \begin{subfigure}[ht]{0.19\linewidth}
      \centering
      \includegraphics[width=\linewidth]{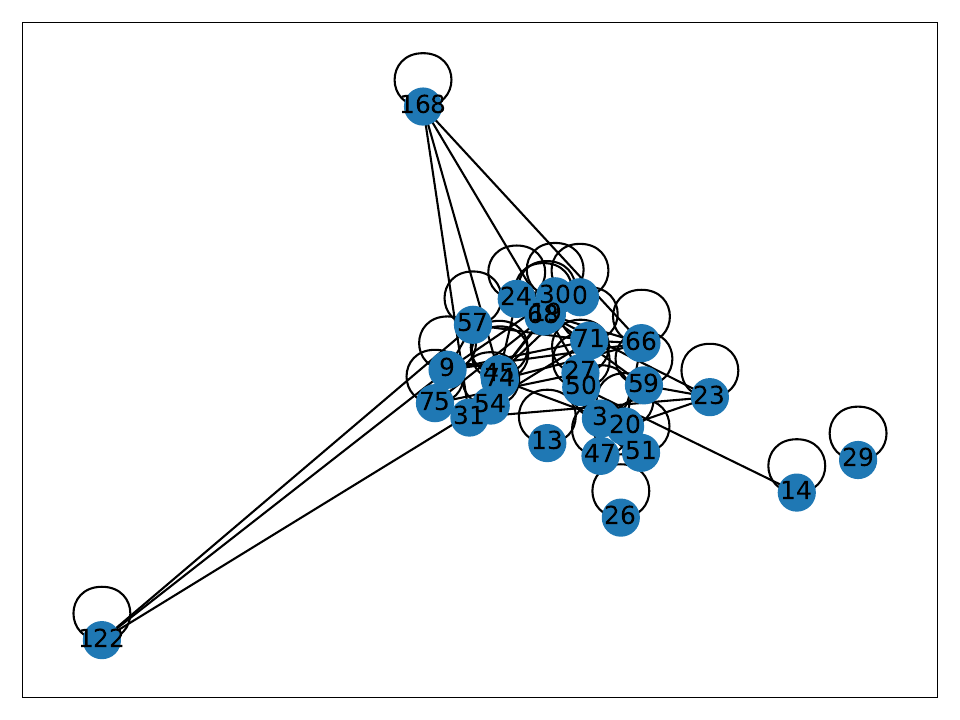}
      \subcaption{Type 5.}
    \end{subfigure}
    \begin{subfigure}[ht]{0.19\linewidth}
      \centering
      \includegraphics[width=\linewidth]{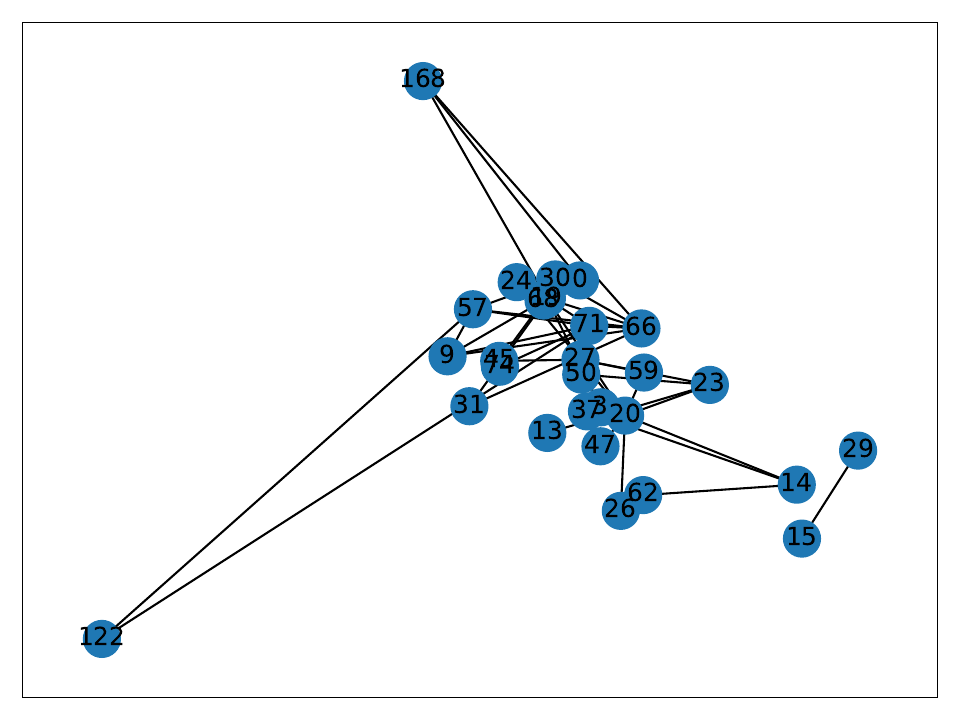}
      \subcaption{Type 6.}
    \end{subfigure}
    \begin{subfigure}[ht]{0.19\linewidth}
      \centering
      \includegraphics[width=\linewidth]{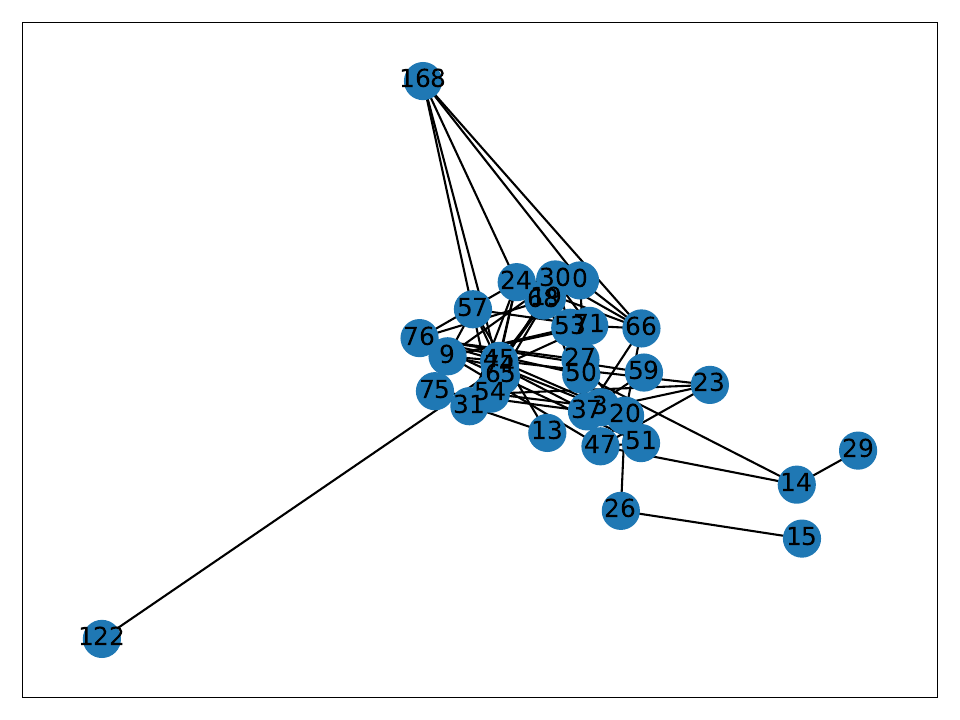}
      \subcaption{Type 7.}
    \end{subfigure}
    \begin{subfigure}[ht]{0.19\linewidth}
      \centering
      \includegraphics[width=\linewidth]{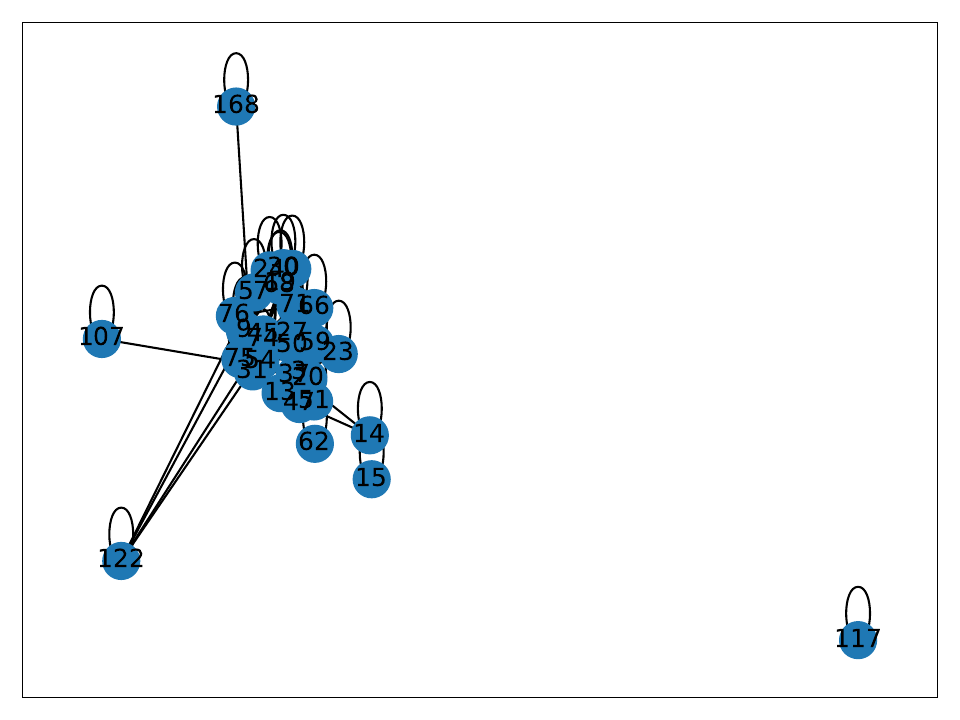}
      \subcaption{Type 8.}
    \end{subfigure}
    \begin{subfigure}[ht]{0.19\linewidth}
      \centering
      \includegraphics[width=\linewidth]{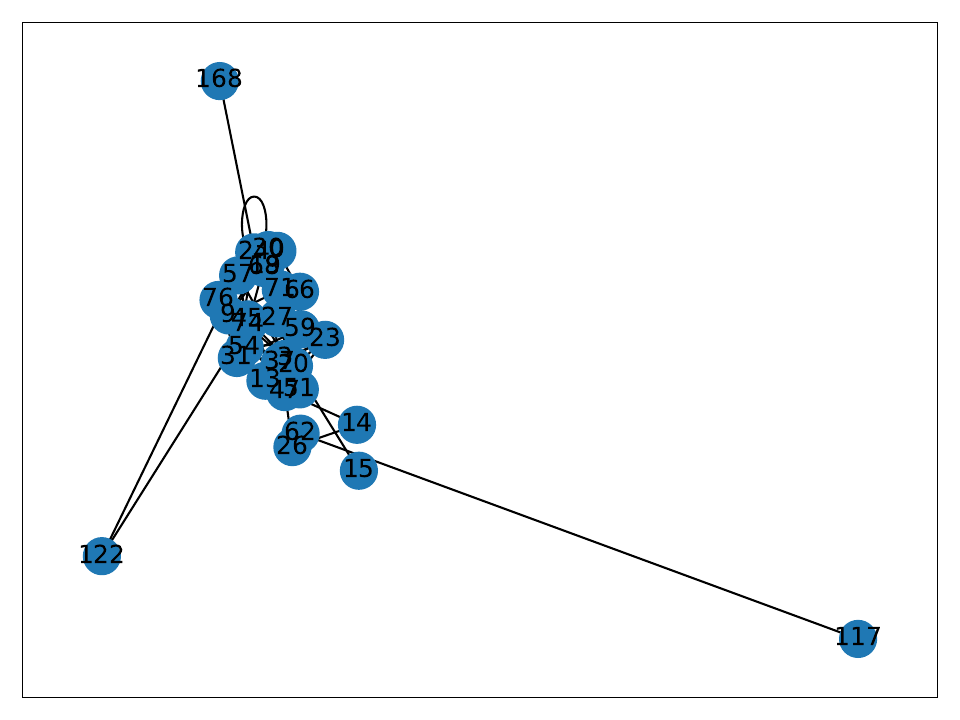}
      \subcaption{Type 9.}
    \end{subfigure}
    \begin{subfigure}[ht]{0.19\linewidth}
      \centering
      \includegraphics[width=\linewidth]{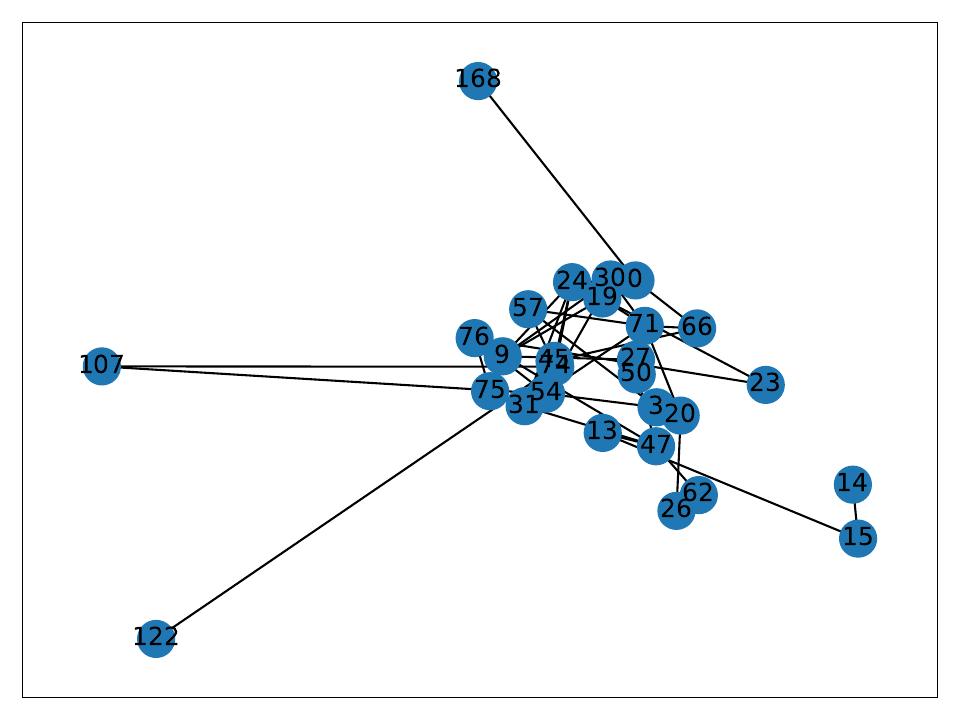}
      \subcaption{Type 10.}
    \end{subfigure}
    \caption{Visualization of the single-relationship networks for August 2006 using the same layout, where only nodes with edges are visualized. Different indices correspond to different individuals in the chimpanzee population.}
    \label{fig:separate_graphs200608}
\end{figure}

\begin{figure}[htb]
\centering
    \begin{subfigure}[ht]{0.6\linewidth}
      \centering
      \includegraphics[width=\linewidth]{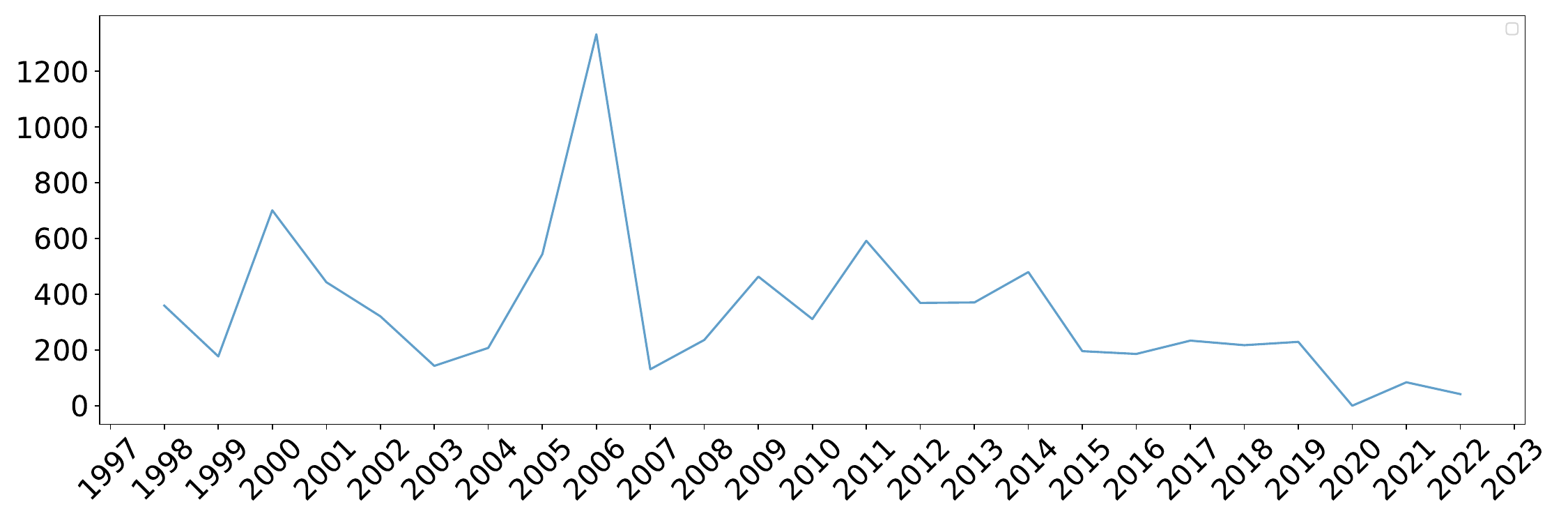}
      \subcaption{Weighted degrees.}
    \end{subfigure}
    \begin{subfigure}[ht]{0.6\linewidth}
      \centering
      \includegraphics[width=\linewidth]{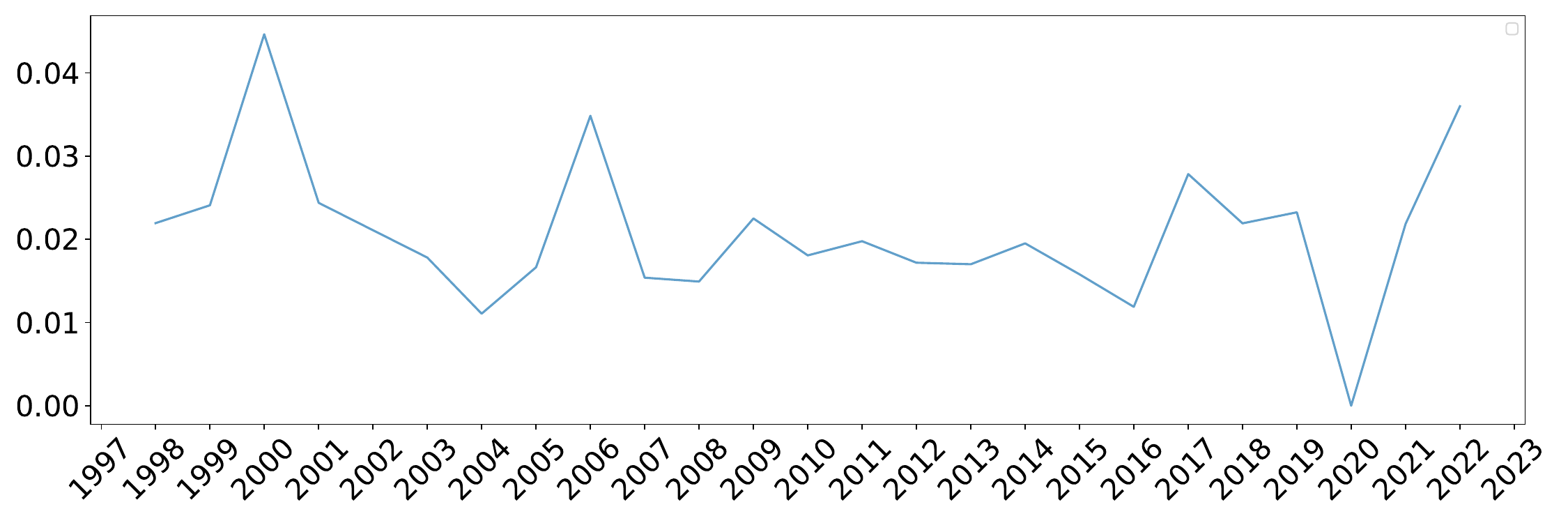}
      \subcaption{Local clustering coefficients.}
    \end{subfigure}
    \begin{subfigure}[ht]{0.6\linewidth}
      \centering
      \includegraphics[width=\linewidth]{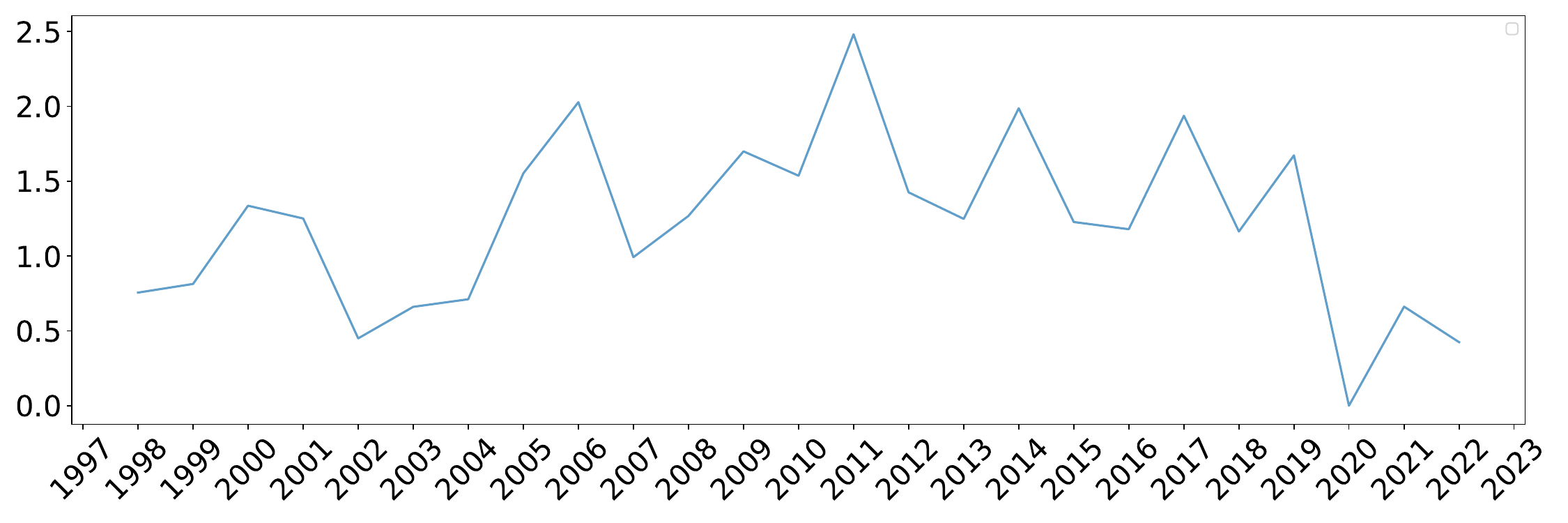}
      \subcaption{Closeness centrality values.}
    \end{subfigure}
    \caption{Key statistics of the learned networks, taking the average over all nodes.}
\label{fig:key_statistics}
\end{figure}

\begin{figure}[htb]
\centering
    \begin{subfigure}[ht]{0.48\linewidth}
      \centering
      \includegraphics[width=\linewidth, trim={0 0 0 1cm}, clip]{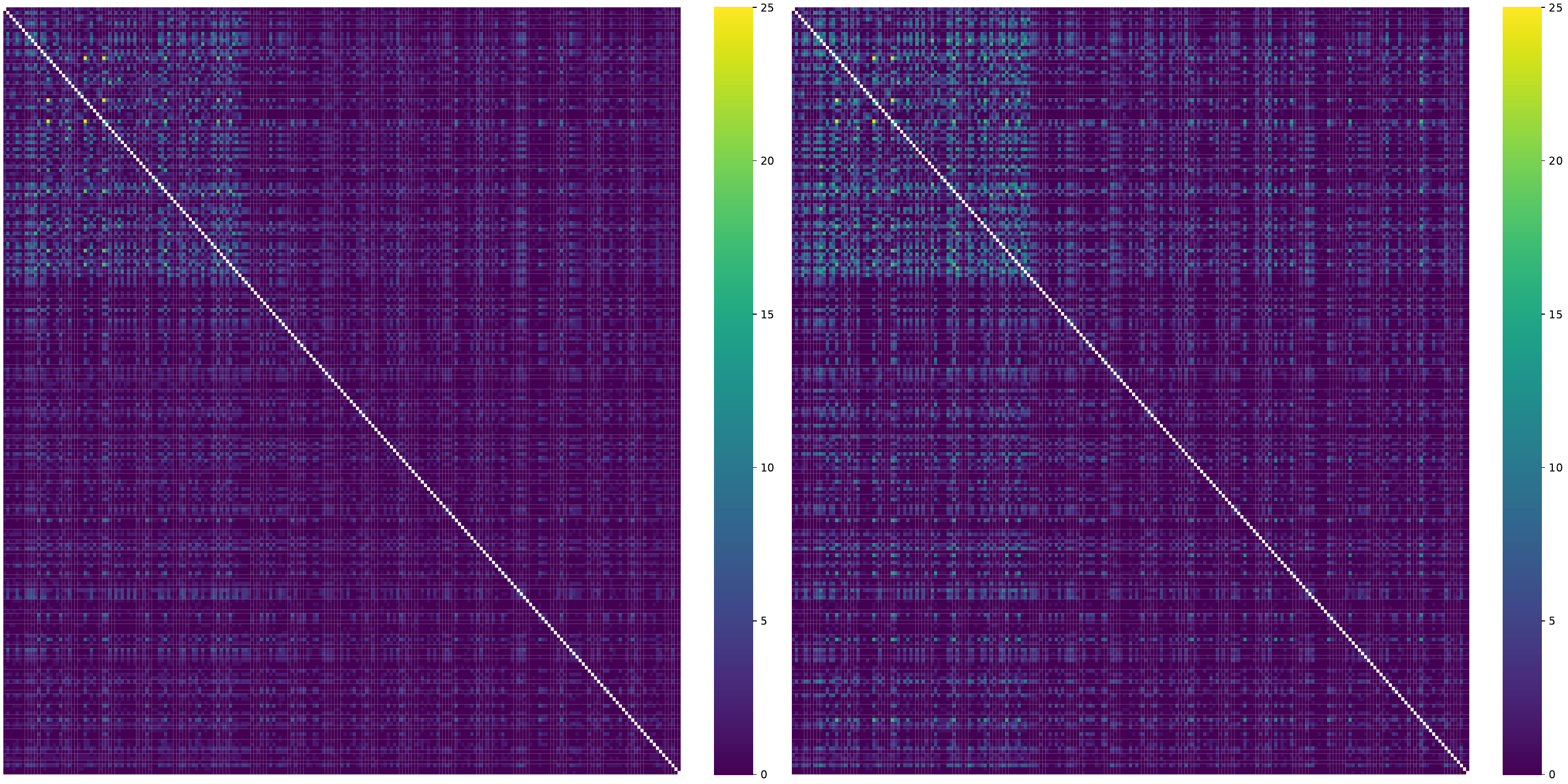}
    \subcaption{Similarity graphs.}
    \end{subfigure}
    \begin{subfigure}[ht]{0.48\linewidth}
      \centering
      \includegraphics[width=\linewidth, trim={0 0 0 1cm}, clip]{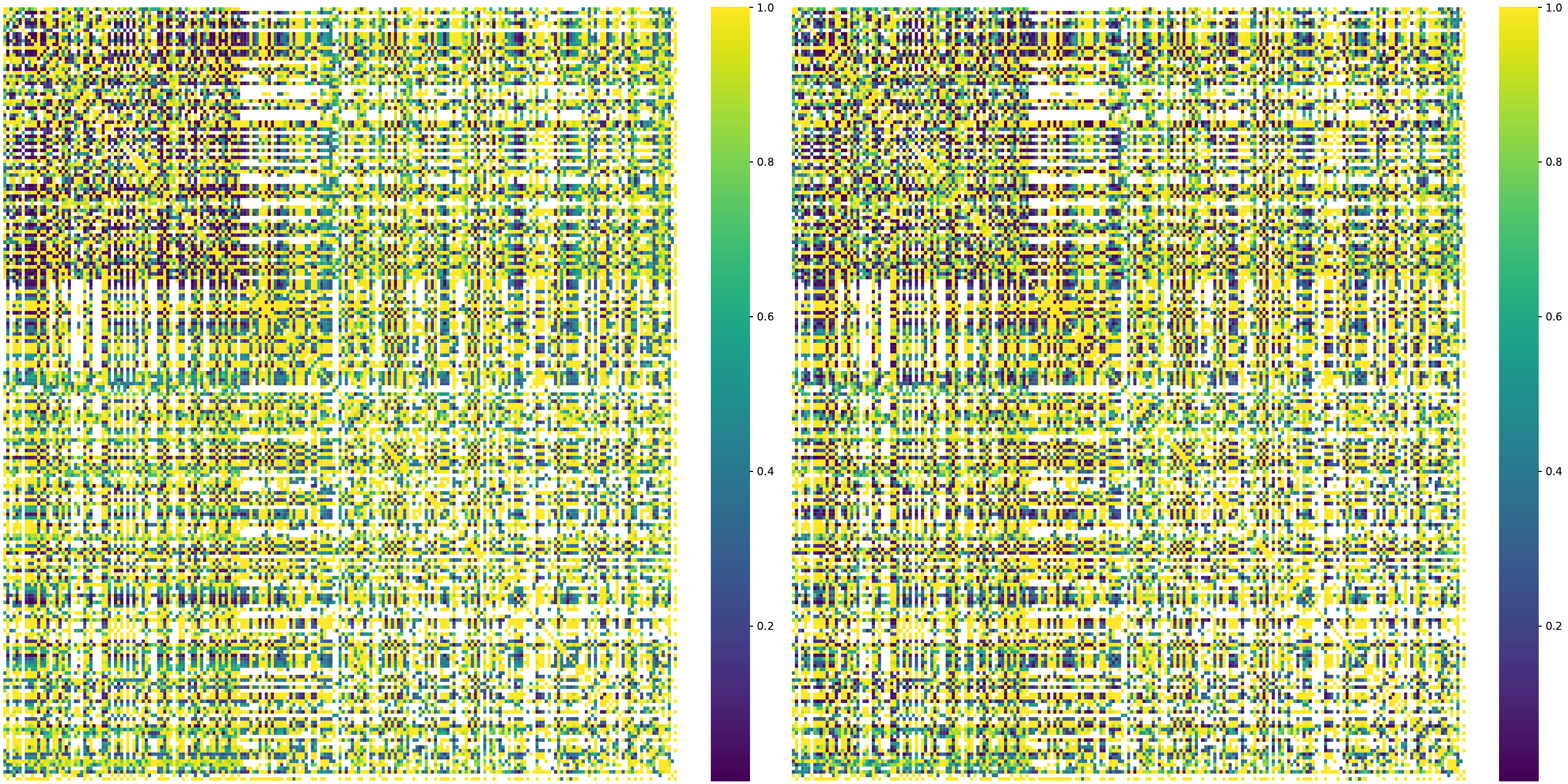}
    \subcaption{$p$-values.}
    \end{subfigure}
    \caption{Similarity graphs and $p$-values for the full learned graphs. For either panel, the left column corresponds to count similarity, while the right column corresponds to duration similarity.}
\label{fig:similarity_and_p_value_full}
\end{figure}
%%%%%%%%%%%%%%%%%%%%%%%%%%%%%%%%%%%%%%%%%%%%%
\begin{table}[htb!]
\centering
\small
\begin{tabular}{lrrrrrrr}
\toprule
Type&$w_2$&$w_3$&$w_4$&$w_5$&$w_\text{add}$ \\ 
\midrule
GT&0.0&1.0&0.0&1.0&0.0\\
$\alpha_3=0.0$&$0.0\pm0.0$&$1.0\pm0.0$&$0.0\pm0.0$&$1.0\pm0.0$&$0.0\pm0.0$&\\
$\alpha_3=0.001$&$0.0\pm0.0$&$0.9\pm0.0$&$0.2\pm0.0$&$0.7\pm0.0$&$0.0\pm0.0$&\\
\midrule
GT&0.0&1.0&0.0&1.0&0.3\\
$\alpha_3=0.0$&$0.0\pm0.0$&$1.0\pm0.0$&$0.0\pm0.0$&$1.0\pm0.0$&$0.3\pm0.0$&\\
$\alpha_3=0.001$&$0.0\pm0.0$&$0.9\pm0.0$&$0.2\pm0.0$&$0.7\pm0.0$&$0.3\pm0.0$&\\
\midrule
GT&0.0&1.0&1.0&0.0&0.0\\
$\alpha_3=0.0$&$0.0\pm0.0$&$1.0\pm0.0$&$1.0\pm0.0$&$0.0\pm0.0$&$0.0\pm0.0$&\\
$\alpha_3=0.001$&$0.0\pm0.0$&$0.9\pm0.0$&$0.8\pm0.0$&$0.1\pm0.0$&$0.0\pm0.0$&\\
\midrule
GT&0.0&1.0&1.0&0.0&0.3\\
$\alpha_3=0.0$&$0.0\pm0.0$&$1.0\pm0.0$&$1.0\pm0.0$&$0.0\pm0.0$&$0.3\pm0.0$&\\
$\alpha_3=0.001$&$0.0\pm0.0$&$0.9\pm0.0$&$0.8\pm0.0$&$0.1\pm0.0$&$0.3\pm0.0$&\\
\midrule
GT&0.0&0.6&0.0&1.2&0.0\\
$\alpha_3=0.0$&$0.0\pm0.0$&$0.6\pm0.0$&$0.0\pm0.0$&$1.2\pm0.0$&$0.0\pm0.0$&\\
$\alpha_3=0.001$&$0.0\pm0.0$&$0.5\pm0.0$&$0.1\pm0.0$&$0.9\pm0.0$&$0.0\pm0.0$&\\
\midrule
GT&0.0&0.6&0.0&1.2&0.3\\
$\alpha_3=0.0$&$0.0\pm0.0$&$0.6\pm0.0$&$0.0\pm0.0$&$1.2\pm0.0$&$0.3\pm0.0$&\\
$\alpha_3=0.001$&$0.0\pm0.0$&$0.5\pm0.0$&$0.1\pm0.0$&$0.9\pm0.0$&$0.3\pm0.0$&\\
\midrule
GT&1.2&0.0&0.0&0.6&0.0\\
$\alpha_3=0.0$&$1.2\pm0.0$&$0.0\pm0.0$&$0.0\pm0.0$&$0.6\pm0.0$&$0.0\pm0.0$&\\
$\alpha_3=0.001$&$0.9\pm0.0$&$0.1\pm0.0$&$0.1\pm0.0$&$0.4\pm0.0$&$0.0\pm0.0$&\\
\midrule
GT&1.2&0.0&0.0&0.6&0.3\\
$\alpha_3=0.0$&$1.2\pm0.0$&$0.0\pm0.0$&$0.0\pm0.0$&$0.6\pm0.0$&$0.3\pm0.0$&\\
$\alpha_3=0.001$&$0.9\pm0.0$&$0.1\pm0.0$&$0.1\pm0.0$&$0.4\pm0.0$&$0.3\pm0.0$&\\
\midrule
GT&0.0&1.0&1.0&1.0&0.0\\
$\alpha_3=0.0$&$0.0\pm0.0$&$1.0\pm0.0$&$1.0\pm0.0$&$1.0\pm0.0$&$0.0\pm0.0$&\\
$\alpha_3=0.001$&$0.0\pm0.0$&$0.9\pm0.0$&$0.9\pm0.0$&$0.7\pm0.0$&$0.0\pm0.0$&\\
\midrule
GT&0.0&1.0&1.0&1.0&0.3\\
$\alpha_3=0.0$&$0.0\pm0.0$&$1.0\pm0.0$&$1.0\pm0.0$&$1.0\pm0.0$&$0.3\pm0.0$&\\
$\alpha_3=0.001$&$0.0\pm0.0$&$0.9\pm0.0$&$0.9\pm0.0$&$0.8\pm0.0$&$0.3\pm0.0$&\\
\midrule
GT&1.0&1.0&1.0&0.0&0.0\\
$\alpha_3=0.0$&$1.0\pm0.0$&$1.0\pm0.0$&$1.0\pm0.0$&$0.0\pm0.0$&$0.0\pm0.0$&\\
$\alpha_3=0.001$&$0.8\pm0.0$&$0.9\pm0.0$&$0.7\pm0.0$&$0.2\pm0.0$&$0.0\pm0.0$&\\
\midrule
GT&1.0&1.0&1.0&0.0&0.3\\
$\alpha_3=0.0$&$1.0\pm0.0$&$1.0\pm0.0$&$1.0\pm0.0$&$0.0\pm0.0$&$0.3\pm0.0$&\\
$\alpha_3=0.001$&$0.8\pm0.0$&$0.9\pm0.0$&$0.7\pm0.0$&$0.2\pm0.0$&$0.3\pm0.0$&\\
\midrule
GT&0.0&0.6&0.3&1.2&0.0\\
$\alpha_3=0.0$&$0.0\pm0.0$&$0.6\pm0.0$&$0.3\pm0.0$&$1.2\pm0.0$&$0.0\pm0.0$&\\
$\alpha_3=0.001$&$0.0\pm0.0$&$0.5\pm0.0$&$0.4\pm0.0$&$0.9\pm0.0$&$0.0\pm0.0$&\\
\midrule
GT&0.0&0.6&0.3&1.2&0.3\\
$\alpha_3=0.0$&$0.0\pm0.0$&$0.6\pm0.0$&$0.3\pm0.0$&$1.2\pm0.0$&$0.3\pm0.0$&\\
$\alpha_3=0.001$&$0.0\pm0.0$&$0.5\pm0.0$&$0.4\pm0.0$&$0.9\pm0.0$&$0.3\pm0.0$&\\
\midrule
GT&0.6&1.2&0.0&0.3&0.0\\
$\alpha_3=0.0$&$0.6\pm0.0$&$1.2\pm0.0$&$0.0\pm0.0$&$0.3\pm0.0$&$0.0\pm0.0$&\\
$\alpha_3=0.001$&$0.5\pm0.0$&$0.9\pm0.0$&$0.2\pm0.0$&$0.2\pm0.0$&$0.0\pm0.0$&\\
\midrule
GT&0.6&1.2&0.0&0.3&0.3\\
$\alpha_3=0.0$&$0.6\pm0.0$&$1.2\pm0.0$&$0.0\pm0.0$&$0.3\pm0.0$&$0.3\pm0.0$&\\
$\alpha_3=0.001$&$0.5\pm0.0$&$0.9\pm0.0$&$0.2\pm0.0$&$0.2\pm0.0$&$0.3\pm0.0$&\\
\bottomrule
\end{tabular}
\caption{Synthetic data results. ``GT" indicates ground truth, $\alpha_3=0$ and $\alpha_3=0.001$ indicate optimized weights by our proposed method for $\alpha_3=0$ and $\alpha_3=0.001$, respectively.
} 
\label{tab:synthetic_res}
\end{table}

\begin{figure}[htb]
\centering
    \begin{subfigure}[ht]{0.32\linewidth}
      \centering
      \includegraphics[width=\linewidth, trim={1cm 1cm 1cm 8cm}, clip]{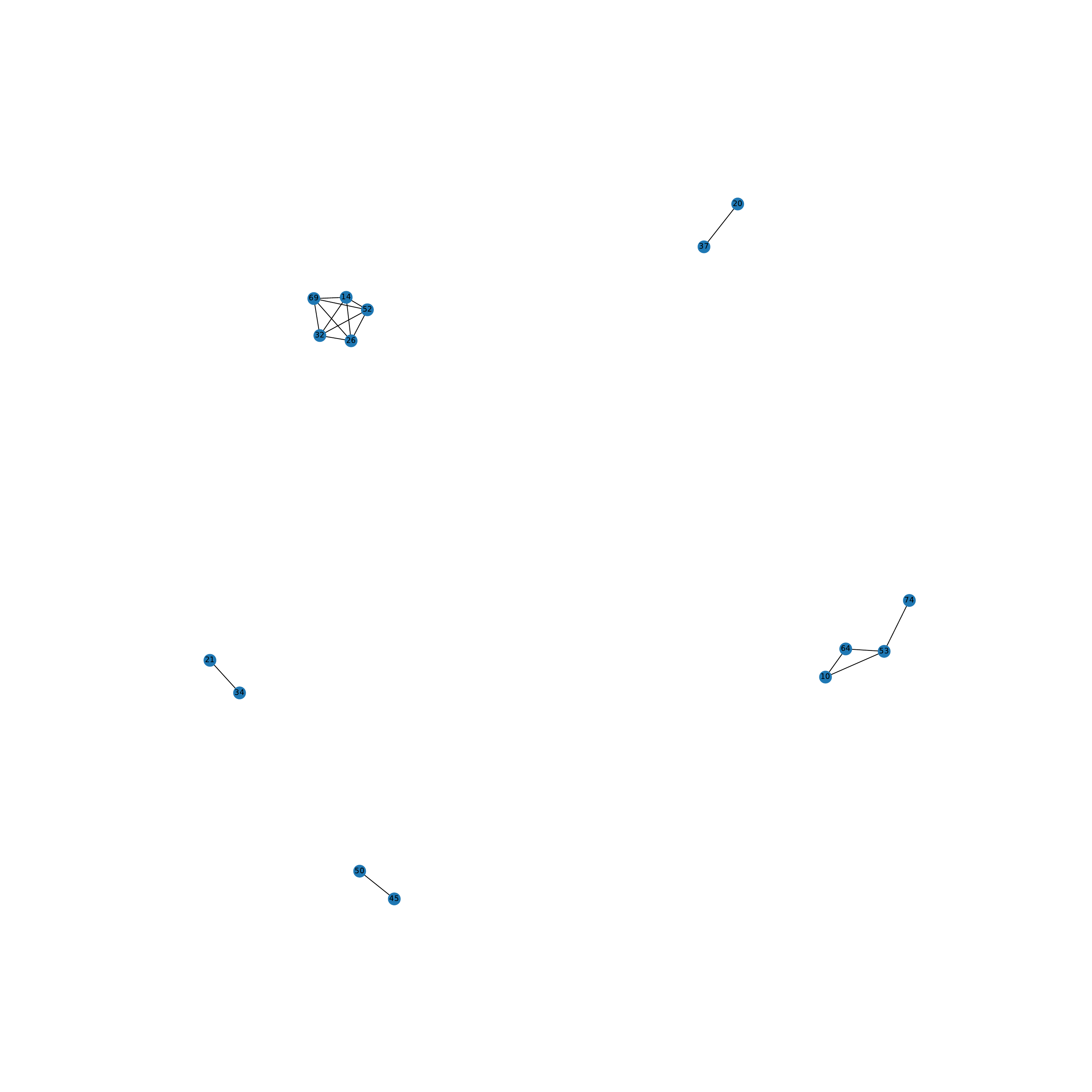}
      \subcaption{Learned: Count Similarity}
    \end{subfigure}
    \begin{subfigure}[ht]{0.32\linewidth}
      \centering
      \includegraphics[width=\linewidth, trim={1cm 1cm 1cm 8cm}, clip]{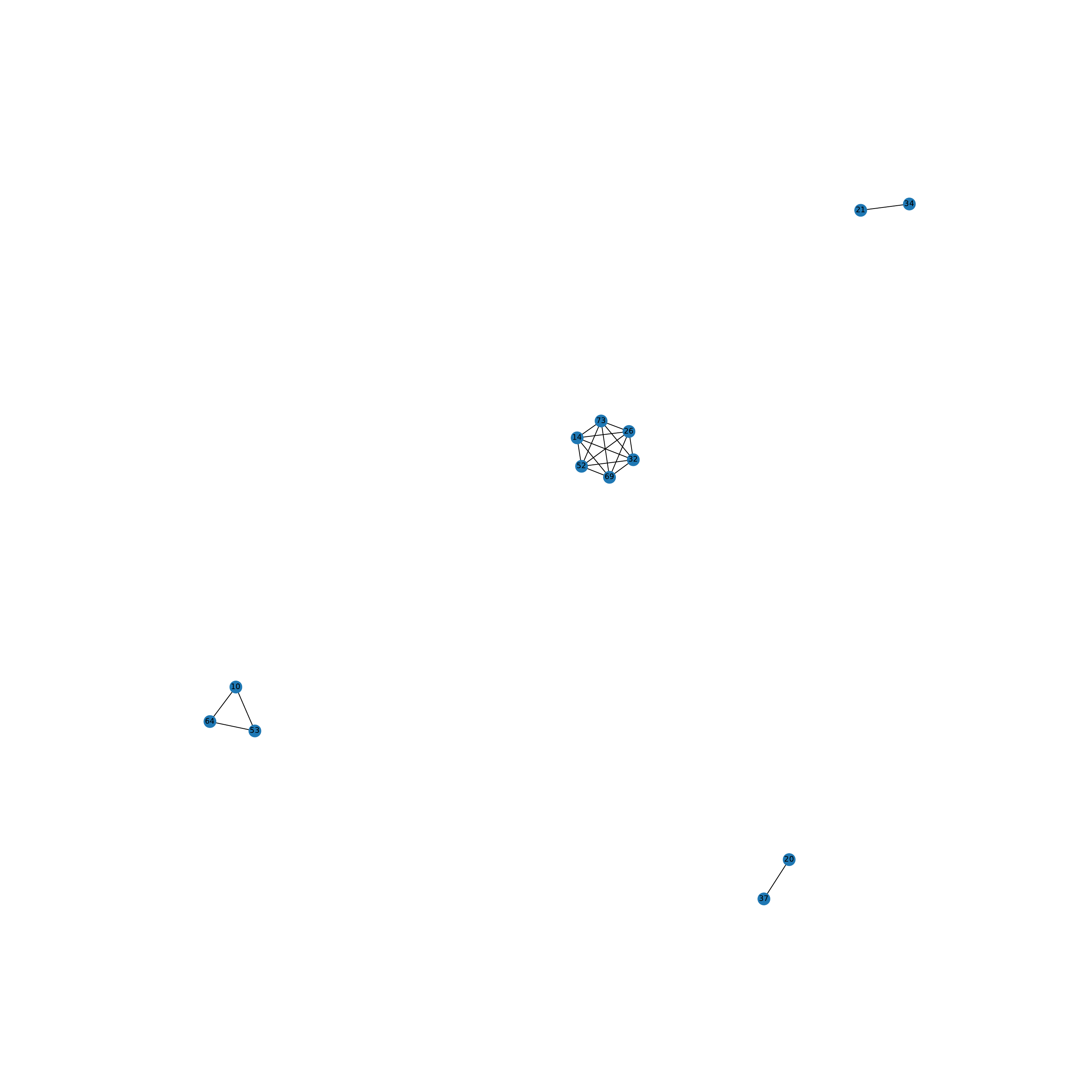}
      \subcaption{Learned: Duration Similarity}
    \end{subfigure}
    \centering
    \begin{subfigure}[ht]{0.32\linewidth}
      \centering
      \includegraphics[width=\linewidth, trim={1cm 1cm 1cm 8cm}, clip]{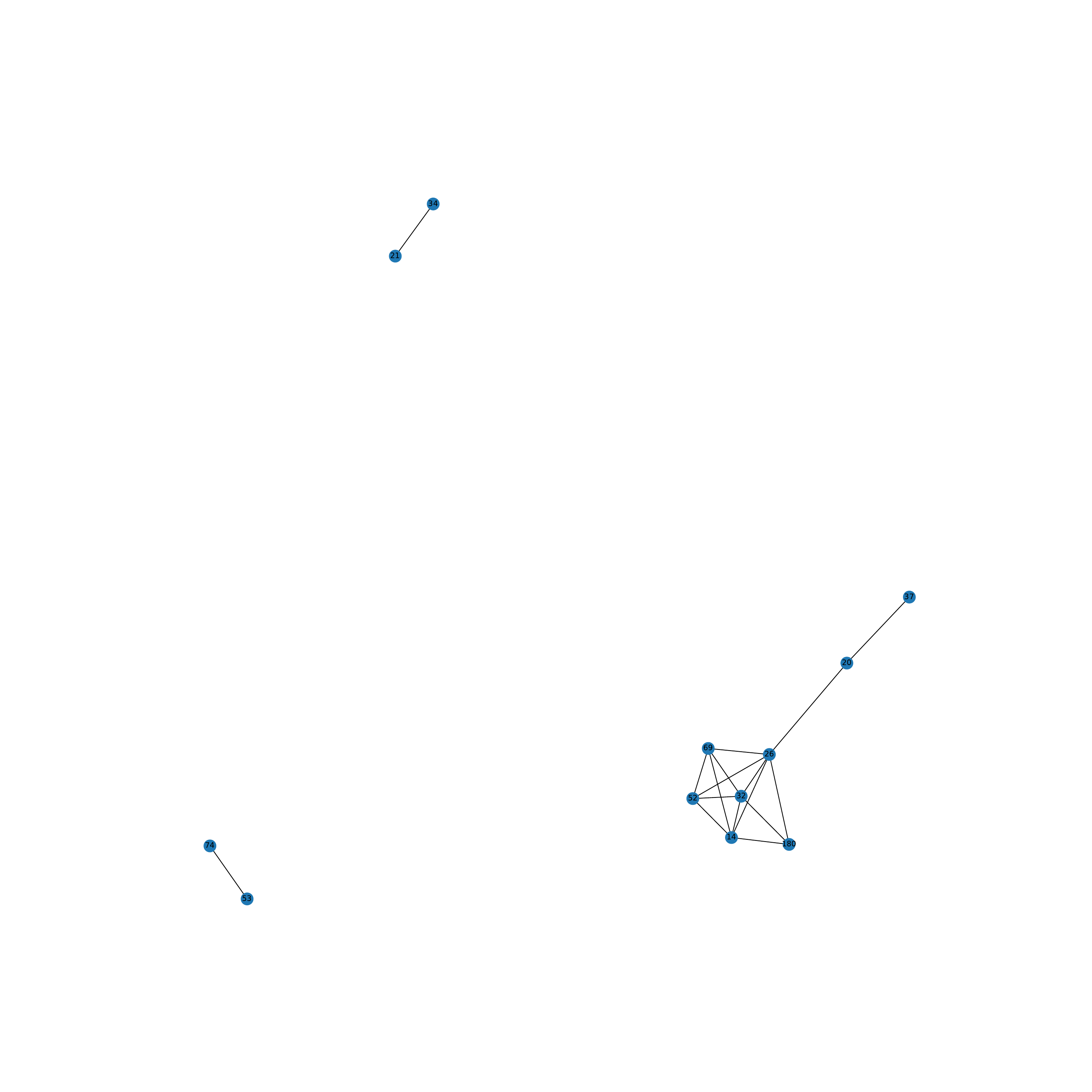}
      \subcaption{Unlearned: Count Similarity}
    \end{subfigure}
    \begin{subfigure}[ht]{0.32\linewidth}
      \centering
      \includegraphics[width=\linewidth, trim={1cm 1cm 1cm 8cm}, clip]{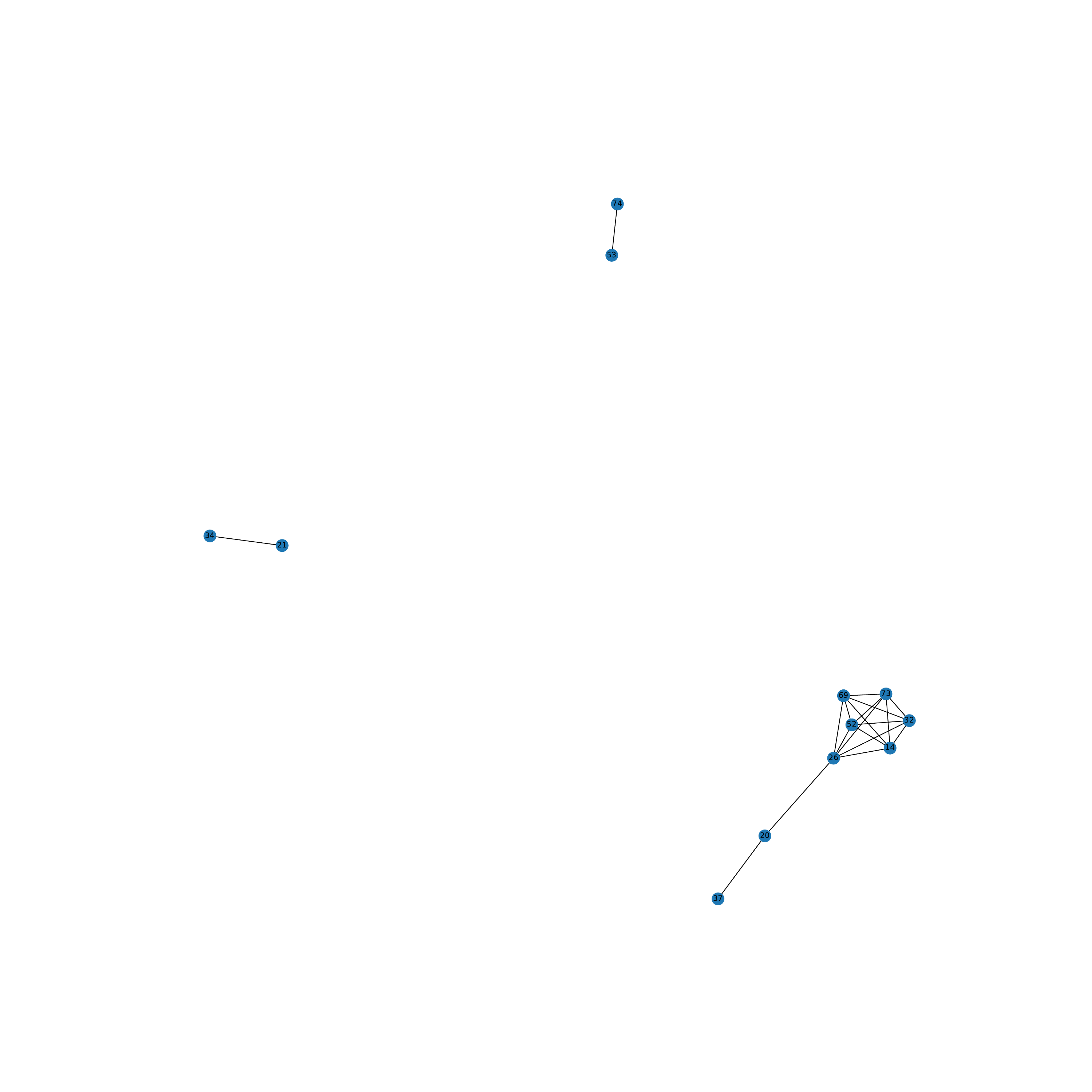}
      \subcaption{Unlearned: Duration Similarity}
    \end{subfigure}
    \begin{subfigure}[ht]{0.32\linewidth}
      \centering
      \includegraphics[width=\linewidth, trim={1cm 1cm 1cm 8cm}, clip]{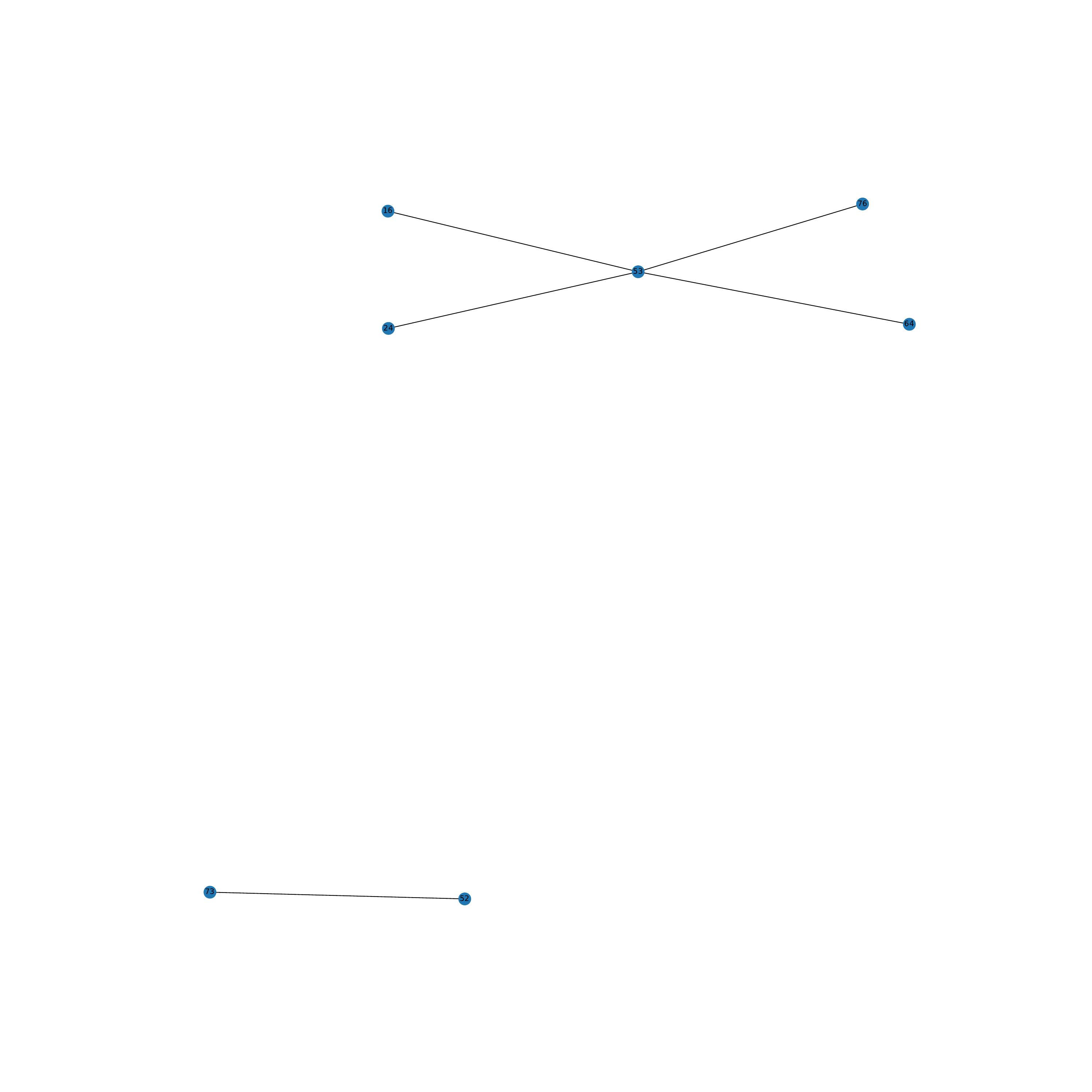}
      \subcaption{Binary: Count Similarity}
    \end{subfigure}
    \begin{subfigure}[ht]{0.32\linewidth}
      \centering
      \includegraphics[width=\linewidth, trim={1cm 1cm 1cm 8cm}, clip]{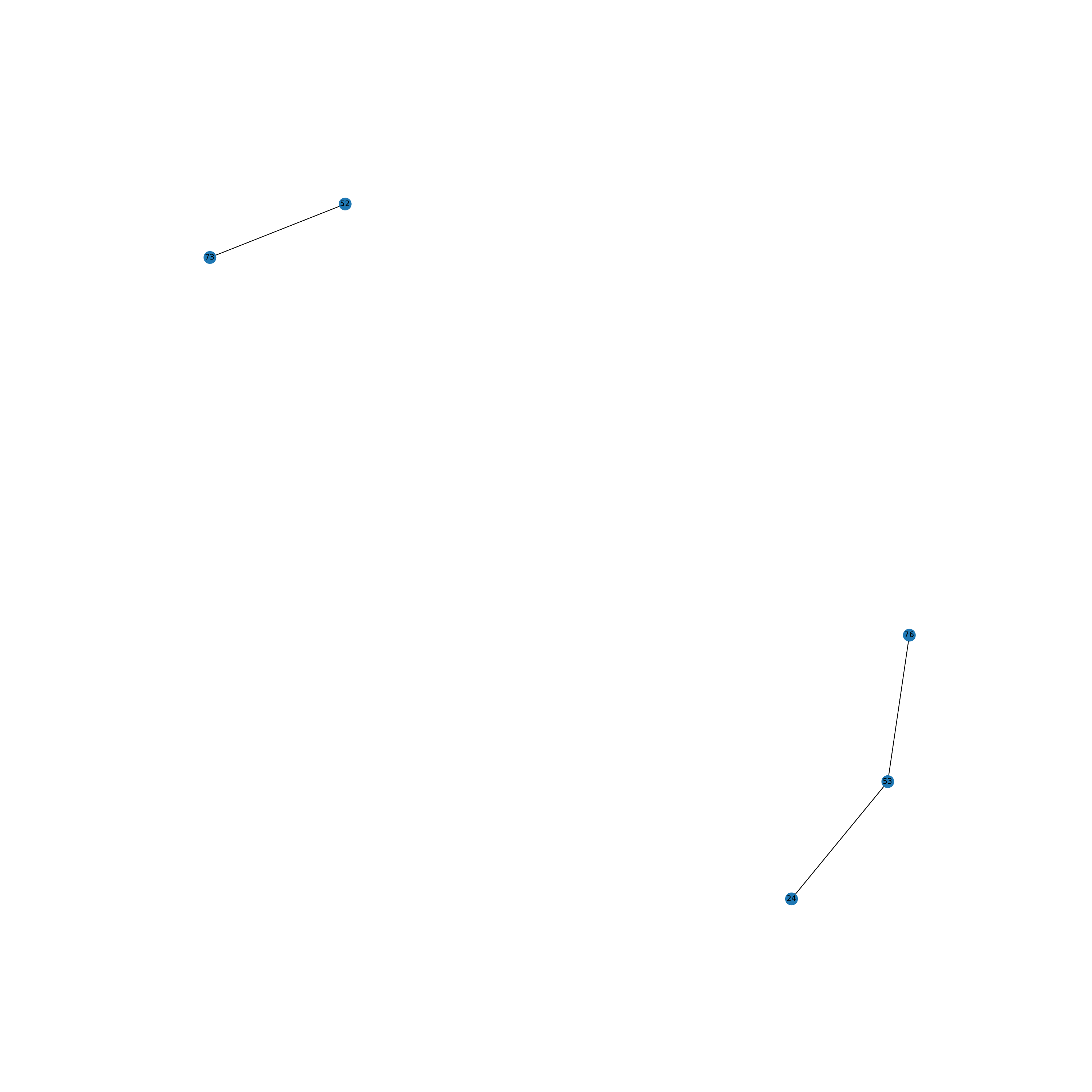}
      \subcaption{Binary: Duration Similarity}
    \end{subfigure}
    \caption{Visualization of thresholded similarity graphs.}
\label{fig:thresholded_similarity}
\end{figure}
%%%%%%%%%%%%%%%%%%%%%%%%%%%%%%%%%%%%%%%%%%%%%%%%%%%%%%%%%%%%%%%%%%%%%%%%%%%%%%%

\end{document}